\def\CLASSINPUTinnersidemargin{1in}
\def\CLASSINPUToutersidemargin{1in}
\def\CLASSINPUTtoptextmargin{1in}
\def\CLASSINPUTbottomtextmargin{1in}
\documentclass[journal,onecolumn,12pt,romanappendices]{IEEEtran}

\usepackage{amsmath,amsfonts,amssymb}
\usepackage{amsthm}
\usepackage{booktabs}
\usepackage{array}
\usepackage{graphicx}
\usepackage{cite}
\usepackage{url}
\usepackage[T1]{fontenc}
\usepackage{setspace}


\newtheorem{theorem}{Theorem}[section]
\newtheorem{lemma}[theorem]{Lemma}
\newtheorem{proposition}[theorem]{Proposition}
\newtheorem{corollary}[theorem]{Corollary}

\newtheorem{assumption}[theorem]{Assumption}

\usepackage{hyperref}
\usepackage{titletoc}
\usepackage{cleveref}
\theoremstyle{definition}
\newtheorem{definition}{Definition}
\usepackage{dsfont}
\usepackage{mathtools}
\usepackage[ruled,linesnumbered]{algorithm2e}

\crefname{equation}{Equation}{Equations}
\crefname{assumption}{Assumption}{Assumptions}
\crefname{definition}{Definition}{Definitions}
\crefname{theorem}{Theorem}{Theorems}
\crefname{lemma}{Lemma}{Lemmas}
\crefname{remark}{Remark}{Remarks}
\crefname{corollary}{Corollary}{Corollaries}
\crefname{algocf}{Algorithm}{Algorithms}

\usepackage{xcolor}
\usepackage{comment}

\begin{document}

\singlespacing
\emergencystretch=1em

\title{Score Mismatch and Set-Constrained Sampling in Zero-Shot Conditional Discrete Diffusion}

\author{Renxiang Huang$^\dagger$, Yuchen Liang$^\dagger$, Yingbin Liang, Ness Shroff, and Lifeng Lai%
\thanks{$\dagger$ These authors contributed equally to this work.

Renxiang Huang and Lifeng Lai are with the Department of Electrical and Computer Engineering, University of California, Davis, Davis, CA, 95616, USA (e-mail: rexhuang@ucdavis.edu; lflai@ucdavis.edu).

Yuchen Liang is with the Luddy School of Informatics, Computing, and Engineering, Indiana University Indianapolis, Indianapolis, IN, 46202, USA (e-mail: yucliang@iu.edu).

Yingbin Liang and Ness Shroff are with the Department of Electrical and Computer Engineering, The Ohio State University, Columbus, OH, 43210, USA (e-mail: liang.889@osu.edu; shroff.11@osu.edu).
}}

\maketitle

\begin{abstract}
    Discrete diffusion models have become a powerful framework for modeling complex discrete distributions across a range of applications. While theoretical convergence guarantees depend on consistency between the target and training distributions, such a case rarely holds in practice, making score mismatch a central challenge in discrete diffusion sampling. This is particularly relevant to zero-shot conditional generation, where a model trained on the unconditional distribution is directly reused at inference time to sample from a conditional target distribution. In this paper, we provide the first theoretical analysis for score-mismatched discrete diffusion models and derive a general KL divergence upper bound that clearly identifies the additional bias caused by such a mismatch. Our analysis covers partial-observation constraints, including text infilling, sentence completion, and more general coordinatewise set constraints. Specifically, we analyze the classical Conditional SEDD sampler and show that it can incur a bias term that scales with the dimension, vocabulary size, and compatibility. Based on this analysis, we propose a Randomized Hard Clamping (RHC) sampler for coordinatewise set constraints. We show that the assignment probabilities can be expressed exactly by the unconditional scores with absorbing rate. Finally, numerical and real-data experiments demonstrate the effectiveness of the proposed RHC sampler.
\end{abstract}



\section{Introduction}
\label{sec:intro}

In recent years, diffusion model has become one of the core paradigms in generative modeling and has achieved impressive success in image, audio, and text generation. Compared with the classical autoregressive (AR) models, diffusion models model the complex data distribution through gradual denoising, demonstrating the notable advantages in generation quality, stability, and flexibility \cite{sohl2015deep,ho2020denoising}. For discrete data such as text sequences, discrete diffusion models extend this denoising principle to the discrete case by defining the noising and denoising processes directly on the discrete state space \cite{austin2021structured,campbell2022discrete}.

In an idealized setting, diffusion models are trained and evaluated on the same underlying distribution. However, in practice, this assumption often fails, as discrepancies between the training and target distributions can arise due to factors such as privacy constraints, computational limitations, and mismatches between the training setup and the intended sampling task \cite{yeganeh2025latent,han2024trainingfree,liang2025theory}. A representative example is conditional generation, where part of the data is observed and the goal is to infer the unobserved part. Many tasks involving discrete data fall into this category, including sequence completion and text infilling \cite{wu2026dreamon,lou2023discrete,varma2025glauber,zhang-etal-2025-flexible}. Such tasks include generating the subsequent content based on the given prefix and filling the missing content while preserving the context. We refer to conditioning on fixed values as sampling ``under fixed observations.'' More generally, each constrained coordinate may admit several candidate values, which leads to sampling under the coordinate wise set constraints. We call a feasible combination of values an ``assignment''. For example, in text rewriting and lexical simplification, each position of the word can be associated with a set of candidate replacements \cite{helbig-etal-2020-challenges,qiang2020lexical}. Fixed observations are recovered when each set contains a single value. Compared to AR models, diffusion models are more naturally suited to these tasks, as they do not rely on a fixed generation order and can flexibly handle partially observed inputs.

Within conditional sampling, one solution is to train a dedicated conditional diffusion model or adapt a pretrained model for each task, such as predictor-free guidance \cite{nisonoff2025unlocking} and inference-time alignment methods \cite{uehara2025inference}. However, such approaches may be costly because they require additional training of a new conditional model beyond unconditional pretraining. In comparison, zero-shot conditional generation does not require such an additional conditional model, since it directly uses the given condition at inference time to build a conditional sampler \cite{lou2023discrete,varma2025glauber}. The core challenge for zero-shot conditional samplers is the inherent mismatch between the pretrained unconditional model and the conditional model needed in the inference process. In diffusion sampling, this discrepancy manifests in the score function, a phenomenon known as \textit{score mismatch} \cite{liang2025theory}.

Among all zero-shot conditional samplers, the Conditional SEDD method \cite{lou2023discrete} is particularly relevant to this paper. At each diffusion sampling step, Conditional SEDD enforces the fixed observations by setting the corresponding components of the current sample to their observed values, an operation we refer to as ``hard clamping.'' Although this method is straightforward to implement in practice, we will show that it is not generally equivalent to the true conditional reverse process, thereby introducing bias. Moreover, this approach still lacks any theoretical justifications. This naturally raises the first question studied in this paper:

\emph{\textbf{Question 1:} Can we characterize the bias induced by \emph{score-mismatched} discrete diffusion models?}

Moving from fixed observations to set constraints introduces another challenge, that is, how to allocate sampling probabilities to each feasible assignment. Once an assignment is selected, the fixed-observation sampler can be reused to generate the remaining contents. After understanding the construction of the assignment distribution, in order to make the fixed-observation sampler better aligned with the target distribution given set constraints while preserving the main advantage of the zero-shot approach, the second question arises:

\emph{\textbf{Question 2:} Can we design a novel method that extends the fixed-observation sampler to the set constraints setting, while explicitly controlling the additional approximation error?}

In this paper, we answer both questions affirmatively.

\subsection{Our Contributions}

Overall, our analysis focuses on how mismatched scores interfere with the reverse rate matrix, thereby introducing bias into the sampling process for both general and zero-shot conditional samplers. This analysis provides the convergence guarantee for each assignment in set constraints. Based on the mixture representation of the assignments, we further extend the analysis from fixed observations to general set constraints. 
The detailed contributions of this paper are:

\textbf{Theoretical Guarantees on General Score-Mismatched Discrete Diffusion Models}: 
Upon characterizing the score mismatch for discrete diffusion models, we provide the \textit{first} non-asymptotic convergence upper bound in KL-divergence between the target (mismatched) and the sampling distributions. Beyond the intrinsic sampling error when there is no mismatch, our theory characterizes a bias term induced by the score mismatch that arises from using the score of the training distribution to sample from a target (mismatched) distribution. Unlike the analysis in the continuous diffusion setting where the mismatch is expressed through the perturbation in the reverse Gaussian mean \cite{liang2025theory}, the mismatch for our setting perturbs the reverse rate matrix of the continuous time Markov chain (CTMC). Moreover, the Bregman divergence term associated with score entropy differs from the usual $L^2$ error for the continuous setting, introducing additional complications in the analysis.

\textbf{Bias Characterization of Conditional SEDD:} 
Based on the general score mismatch analysis, we further provide the \textit{first} theoretical guarantees for the classical Conditional SEDD algorithm \cite{lou2023discrete} under a fixed observation. By comparing the hard clamp score with the true conditional score, we reveal that the effect of hard clamping depends on the diffusion process. Under the absorbing rate, these two scores are identical on unobserved coordinates given the fixed non-MASK observation, and thus the convergence bound does not contain an additional score-mismatch bias term. Under the uniform rate, the hard clamping can introduce score mismatch over unobserved coordinates. With the assumption of condition compatibility, we bound this bias and characterize the dependence on the unobserved coordinates, vocabulary size, and compatibility parameter.

\textbf{Randomized Hard Clamping (RHC) for Set Constraints:}
We propose a Randomized Hard Clamping (RHC) sampler for the set-constraint setting. In particular, we consider the setting where each constrained index in $\bar\Omega$ admits several candidate values from the admissible sets rather than a single fixed value. As a result, the target conditional distribution admits an exact mixture representation according to the assignments. We further show that the exact assignment distribution can be reconstructed based on the unconditional score for absorbing-rate CTMCs. As a result, the proposed RHC sampler achieves zero-shot generation with a pretrained unconditional discrete diffusion model. The numerical simulations and real-data experiments further support our theoretical results and validate the practical effectiveness of the proposed RHC sampler.

\subsection{Related Works}

In this section, we focus on related works for convergence analyses and conditional generation for discrete diffusion models. Other related works on conditional continuous diffusion are provided in \Cref{app:relaconti}.

\textbf{Score-matched Discrete Diffusion Model:} The existing discrete diffusion analysis focuses on the score-matched setting, that is, the target distribution is consistent with the training distribution. In particular, several recent works have established non-asymptotic convergence guarantees under different noise and sampling schemes \cite{liang2025absorb,liang2025discrete,chen2025convergence,ren2025stoc-int,zhang2025convergence}. However, these results cannot be directly applied when the target distribution differs from the training distribution, as in the zero-shot conditional generation. 

\textbf{Guided Discrete Diffusion Model:} Recent works have studied guidance mechanisms in discrete diffusion. \cite{nisonoff2025unlocking} proposed predictor guidance and predictor-free guidance applicable to CTMC-based discrete diffusion. Their method realized the guidance through scaling the rate matrix. \cite{schiff2025simple} directly derived the classifier-based and classifier-free guidance rules by reweighting the transition kernel in discrete diffusion. 

For conditional sampling at large, there are several strands of work. \cite{kerby2024training} proposed a discrete training-free guidance framework. \cite{lin2025tfgflow} extended the training-free guidance to multimodal data. \cite{kit2025debiasing} pointed out that the bias exists in common discrete guidance methods and adopted Sequential Monte Carlo for debiasing. In addition, \cite{wang2025finetuning} studied reward optimization in the discrete diffusion. Similarly, \cite{li2025derivativefree} proposed soft value-based decoding, which introduced the soft value function to guide the model to approximate the optimal policy to implement conditional generation. 

These studies focus on how to incorporate conditional information into the sampling process. Our work provides a general convergence framework to explicitly quantify the additional sampling error introduced by the mismatch.

\textbf{Zero-Shot Discrete Diffusion Model:} Zero-shot sampling has recently been explored in the discrete diffusion model. SEDD and RADD support zero-shot text infilling by fixing the prompt at each sampling step instead of retraining for the condition \cite{lou2023discrete,ou2024your}. A discrete diffusion model based on heat bath dynamics, developed by \cite{varma2025glauber}, reported good performance in zero-shot tasks, such as text and image infilling. \cite{liu2026trainingfree} proposed a training-free method to calibrate the confidence of the sequence length to allow the adaptive length of the generated sequence. \cite{chu2025split} proposed a plug-and-play posterior sampling method based on split Gibbs sampling. Their method introduces external conditions through the likelihood or reward guidance. Finally, note that all the works mentioned above focused mainly on algorithmic design and empirical performance, but did not characterize the theoretical performance introduced by modifications to the sampling process. 

Our work establishes non-asymptotic convergence guarantees to quantify how the score mismatch affects the convergence of zero-shot conditional sampling with absorbing and uniform rates. This framework further supports the analysis for set constraints.

\textbf{Constrained Diffusion Models:} There are several studies regarding the explicit constraint for diffusion models. In continuous diffusion, \cite{christopher2024constrained} formulate the sampling process as an optimization problem to guide the generation result to satisfy the constraints. Using a similar idea, \cite{cardei2025constrained} introduced Constrained Discrete Diffusion. They incorporated differentiable projection and the augmented Lagrangian method into the sampling process so that the generated samples satisfy the conditional constraints. On the application side, \cite{mundler2026constrained} proposed a constrained decoding method applicable to the diffusion language model, which aims to make the generation result conform to the context-free grammar. 

For set constraints, our RHC uses the unconditional score to assign probabilities to feasible values. The convergence guarantee separates the assignment distribution error and the sampling error.

\section{Problem Setup}
\label{sec:probsetup}

In this section, we first review the basic concepts of score-matched discrete diffusion models. Then we extend it to general score-mismatched discrete diffusion models. Finally, we specialize this framework to conditional generation problems and zero-shot methods. 

\subsection{(Score-matched) Discrete Diffusion Models}

We consider the diffusion model defined on the discrete state space $\mathcal{X}=[S]^d$, where $d$ denotes the length or dimensions of the sequence and $S$ denotes the number of possible values for each dimension. Let $q_0$ be the training data distribution and $x_0\sim q_0$ be a data sample. The training distribution has full support on the clean state space. The basic pipeline of the discrete diffusion model is to first gradually perturb the data distribution into a simple stationary distribution, and then learn the corresponding reverse process for denoising.

\textbf{Forward process:} Let $T$ be the time horizon and $q_t$ be the marginal distribution for $0<t<T$. The forward process can be expressed as a CTMC with the transition probability in a small time interval $[t-dt,t]$ given by
\begin{equation*}
    q_{t|t-dt}(y|x)=\mathds{1}_{y=x}+Q_t(x,y)dt+o(dt),
\end{equation*}
where $\mathds{1}_{y=x}$ denotes the indicator function that equals $1$ only when $y=x$ and $0$ otherwise, and $Q_t\in\mathbb{R}^{S^d\times S^d}$ refers to the rate matrix of the CTMC. The rate matrix $Q_t$ controls the instantaneous rate of the state transition. A valid rate matrix $Q_t$ should satisfy the following requirements: $Q_t(x,y)\geq0$ for $x\neq y$, $Q_t(x,x)\leq0$ and also $Q_t(x,x)=-\sum_{y\neq x}Q_t(x,y)$. To simplify the computation, a common assumption is that the forward propagation of each token $x^i\in[S]$ is independent \cite{lou2023discrete,campbell2022discrete,liang2025absorb,liang2025discrete}. This independence is also reflected in the rate matrix $Q_t$, that is, for $x\neq y$ only at one index $i$, $Q_t(x,y)=Q^{\text{token}}_t(x^i,y^i)$. Here, $Q^{\text{token}}_t\in \mathbb{R}^{S\times S}$ is the rate matrix of each token. In this work, we focus on the time-homogeneous case, which means that the rate matrix $Q_t$ is independent of time $t$ during the entire forward process, i.e., $Q_t\equiv Q,~Q^{\text{token}}_t\equiv Q^{\text{token}}$. 

There are two widely studied rate matrices in previous research. The first one is the absorbing rate matrix \cite{campbell2022discrete,lou2023discrete,liang2025absorb,ou2024your}. An absorbing rate matrix will cause all tokens to enter a MASK state at the terminal time point. Given $\text{MASK}\in [S]$, the absorbing rate matrix for one token is defined as
\begin{equation}\label{abrate}
    Q^{\text{token}}=\mathbf{1}_{S} e_{\text{MASK}}^\top-I_S,
\end{equation}
where $\mathbf{1}_{S}$ is a \textit{vector} of ones with a length of $S$, $e_{\text{MASK}}$ is a unit vector where the element is $1$ only at the MASK position, and $I_S$ is an identity matrix with a size of $S\times S$. By this definition, for all $x\in [S],~x\neq \text{MASK}$, the diagonal elements satisfy $Q_t^{\text{token}}(x,x)=-1$ and the elements in the column of MASK state satisfy $Q_t^{\text{token}}(x,\text{MASK})=1$. For the other cases, the elements in $Q_t^{\text{token}}$ will be 0. Based on this definition, we can have
$Q(x,x)=-\sum_{y\neq x}Q(x,y)=-(d-m(x))$, where $m(x)$ is defined as the number of MASK tokens in $x$. Here, we also emphasize that the clean data do not contain the MASK token throughout the subsequent analysis, whereas the diffusion state space includes MASK.

The second one is the uniform rate matrix \cite{campbell2022discrete,liang2025discrete,chen2025convergence}. A uniform rate matrix will cause all tokens to be distributed uniformly at the terminal time point. The token-wise rate matrix is defined as
\begin{equation}\label{unirate}
    Q^{\text{token}}=\frac{1}{S}\mathbf{1}_S\mathbf{1}^\top_S-I_S.
\end{equation}
This ensures that $Q$ satisfies the requirements of a valid rate matrix. Hence, we can derive $Q(x,x)=-\sum_{y\neq x}Q(x,y)=-d\frac{S-1}{S}$.

\textbf{Reverse process:} Given a forward CTMC, the time-reverse process is also a CTMC and can be described by the time-reverse process $\overleftarrow{q}_t=q_{T-t}$. The reverse rate matrix $\overleftarrow{Q}_t(x,y)$ for $x\neq y$ satisfies \cite[Proposition~1]{campbell2022discrete}
\begin{equation}\label{eq:revq}
    \overleftarrow{Q}_t(x,y)=Q_{T-t}(y,x)\frac{q_{T-t}(y)}{q_{T-t}(x)}.
\end{equation}
Similarly, for the reverse rate matrix, we also have $\overleftarrow{Q}_t(x,y)\geq0$ for $x\neq y$, $\overleftarrow{Q}_t(x,x)\leq0$, and $\overleftarrow{Q}_t(x,x)=-\sum_{y\neq x}\overleftarrow{Q}_t(x,y)$. Define $s_t(y,x):=\frac{q_t(y)}{q_t(x)}$ as the score function in discrete diffusion. Different from the score function $\nabla\log q_t(x)$ in continuous diffusion, the discrete score measures the probability ratio of state $y$ relative to the state $x$ at time point $t$. Thus, the reverse rate matrix can be written as $\overleftarrow{Q}_t(x,y)=Q_{T-t}(y,x)s_{T-t}(y,x)$. Since the true score function $s_t(y,x)=\frac{q_t(y)}{q_t(x)}$ is inaccessible, we train a neural network to learn this approximation $\hat{s}_t \approx s_t$.

\textbf{Training Objective:} Since the true distribution $q_t$ cannot generally be obtained explicitly, one can introduce an approximation $\hat{s}_t$ to approach the true score $s_t$ such that $\hat{Q}_t(x,y)=Q_{T-t}(y,x)\hat{s}_{T-t}(y,x)$. According to \cite{lou2023discrete}, a natural training objective is to minimize the score entropy loss:
\begin{equation}\label{eq:se}
    \mathcal{L}_{SE}=\mathbb{E}_{x_t\sim \overleftarrow{q}_t}\left[\sum_{y\neq x_t}\left(\hat{Q}_t(x_t,y)-\overleftarrow{Q}_t(x_t,y)+\overleftarrow{Q}_t(x_t,y)\log \frac{\overleftarrow{Q_t}(x_t,y)}{\hat{Q}_t(x_t,y)}\right)\right].
\end{equation}

In this work, we focus on two practical samplers, $\tau$-leaping \cite{campbell2022discrete} and uniformization \cite{van1992uniformization}. $\tau$-leaping is a time-discretized sampling method. With early-stopping, given the discrete time grids $0=t_0<t_1<\cdots<t_N=T-\delta$, the reverse rate matrix in each time interval $[t_k,t_{k+1})$ will be frozen as $\hat{Q}_{t_k}$. Given the current state $x_{t_k}$, the number of the state transitions in $[t_k,t_{k+1})$ follows a Poisson distribution with rate $\hat{Q}_{t_k}(x_{t_k},x_{t_k}^{-i}+s\cdot e^i)(t_{k+1}-t_{k})$, where $e^i$ is a unit vector with $1$ only at the $i$-th position. The update rule is thus given by
\begin{equation}
    \label{def:tau-leaping}
    x_{t_{k+1}}=x_{t_k}+\sum_{i=1}^d\sum_{s=1}^S(s-x_{t_k}^i)\text{Pois}(\hat{Q}_{t_k}(x_{t_k},x_{t_k}^{-i}+s\cdot e^i)(t_{k+1}-t_{k}))\cdot e^i.
\end{equation}
In particular, for $\tau$-leaping, one update includes all possible transitions during $[t_k,t_{k+1})$. For discrete state spaces, $\tau$-leaping may suffer from an out-of-boundary issue. One can handle this issue by clipping in practice \cite[Remark~A.13]{ren2025stoc-int}. The final convergence is not affected when the time step is sufficiently small \cite[Proposition~A.14]{ren2025stoc-int}. On the other hand, uniformization is a continuous-time sampling method \cite{van1992uniformization}. For $t\in[t_k,t_{k+1})$, let the rate $\lambda_k$ satisfy $\lambda_k\geq\sup_{x\in [S]^d} \sum_{y\neq x}\hat{Q}_t(x,y)$, then the transition from the current state $x$ to $y$ will occur with a probability of $\hat{Q}_t(x,y)/\lambda_k$. Notably, the uniformization has been proven to exactly simulate the time-inhomogeneous CTMC \cite{chen2025convergence,ren2025stoc-int}.

\subsection{Score Mismatch in Discrete Diffusion Models}

For the score-mismatched setting, the \textit{target} distribution that we expect to sample from is not $q_0$ used for training, but another distribution $\tilde{q}_0$. With the same forward rate matrix $Q_t$, the corresponding forward marginal distribution is $\tilde{q}_t$ ($t>0$) and the time-reverse distribution is $\overleftarrow{\tilde{q}}_t=\tilde{q}_{T-t}$. The reverse rate matrix $\overleftarrow{\tilde{Q}}_t$ corresponding to the target distribution is given by \begin{equation*}
\overleftarrow{\tilde{Q}}_t(x,y)=Q_{T-t}(y,x)\frac{\tilde{q}_{T-t}(y)}{\tilde{q}_{T-t}(x)}\eqqcolon Q_{T-t}(y,x)\tilde{s}_{T-t}(y,x).
\end{equation*}
To quantify the deviation resulting from the inconsistency between the target score $\tilde{s}_t$ and training score $s_t$, we introduce the score mismatch $\Delta_t$ as follows:
\begin{definition}[Discrete Score Mismatch] \label{def1}
    Given the target distribution $\tilde{q}_0$ and the corresponding forward process $\tilde{q}_t$, the score mismatch $\Delta_t(y,x)$ between $\tilde{s}_t$ and $s_t$ is defined as
    \begin{equation*}
        \Delta_t(y,x):=\frac{\tilde{s}_{T-t}(y,x)}{s_{T-t}(y,x)} - 1 = \left(\frac{\tilde{q}_{T-t}(y)}{\tilde{q}_{T-t}(x)}\right)\Big/\left(\frac{q_{T-t}(y)}{q_{T-t}(x)}\right)-1,\quad \forall x\neq y,
    \end{equation*}
    where $\Delta_t(y,x)\geq-1$
    for $y\neq x$ and $\Delta_t(x,x)=0$.
    Hence, for $x\neq y$, we have $\overleftarrow{\tilde{Q}}_t(x,y)=\overleftarrow{Q}_t(x,y)(1+\Delta_t(y,x))$,
    where $\overleftarrow{\tilde{Q}}(x,x)=-\sum_{y\neq x}\overleftarrow{\tilde{Q}}(x,y)$.
\end{definition}

The score mismatch is naturally a distortion of the original reverse jump rate. This is an important difference from the additive score mismatch, which is typical for continuous diffusion models \cite{liang2025theory}. For the continuous models, the score mismatch can be expressed as an additive perturbation in the Gaussian reverse mean, that is, the difference between the posterior means. In contrast, for discrete diffusion models, the reverse process is driven by the rate matrix, where the change of the target distribution will result in a modification of these \textit{ratios} rather than an additive perturbation. This motivates defining score mismatch also in terms of a ratio.

\subsection{Zero-Shot Conditional Sampling under Fixed Observations}\label{sec:zs}

In the above general score mismatch setting, the target distribution $\tilde{q}_0$ differs from the training distribution $q_0$. A crucial scenario of this setting is conditional generation, such as text infilling \cite{lou2023discrete,ou2024your}. For this scenario, we do not aim to sample from the unconditional distribution $q_0$; instead, our goal is to complete the unobserved parts of a sequence given the condition of observing the partial sequence. More specifically, let $x_0\in [S]^d$ be divided into two parts: $x_0=[x_0^{\Omega},x_0^{\bar\Omega}]$, where $\Omega$ refers to the set of unobserved indices that need to be filled and $\bar\Omega$ refers to the observed indices. Given a non-MASK observation $y\in [S]^{|\bar\Omega|}$, the target distribution is the conditional distribution of the unobserved coordinates given the observation $y$: 
\[\tilde{q}_0= q_{0|y}^\Omega=q_0(x_0^\Omega|x_0^{\bar\Omega}=y).\]
In the zero-shot setting, our goal is to sample from the conditional target using a pretrained unconditional discrete diffusion without any additional training for the given observations.

\subsection{Zero-Shot Conditional Sampling under Set Constraints}\label{sec:pbset:setcons}
More generally, $\bar\Omega$ can refer to constrained coordinates, where each coordinate in $\bar\Omega$ may accept multiple feasible values. For each coordinate $m\in \bar\Omega$, given the nonempty set $A_m\subseteq [S]\backslash{\mathrm{MASK}}$, we define the set constraints as $C=\{x_0^m\in A_m,m\in\bar\Omega\}$ and the feasible set as $A=\prod_{m\in\bar\Omega}A_m$, then equivalently, $C=\{x_0^{\bar\Omega}\in A\}$. Then, the corresponding target distribution should be the conditional marginal given by\[\tilde{q}_0= q_{0|C}^\Omega=q_0(x_0^\Omega|C),\] and the fixed observation will be recovered if $A_m=\{y_m\}$ for each $m\in \bar\Omega$. We extend the same zero-shot setting to these set constraints, aiming to sample from $q_{0|C}^\Omega$ using a pretrained unconditional discrete diffusion without any additional training for the specified constraints.

\section{Guarantees for General Score-Mismatched Discrete Diffusion Models}\label{sec:gc}

In this section, we present the results for the convergence analysis of general score-mismatched discrete diffusion model for a general target distribution $\tilde{q}_0$.
We first analyze the $\tau$-leaping sampler and adopt the following two assumptions, which are commonly used in previous studies for the score-matched discrete diffusion models \cite{liang2025absorb,liang2025discrete,ren2025stoc-int}.
\begin{assumption}\label{assp1} (Score Estimation Error). The score estimation $\hat{s}_{T-t_k}$ satisfies \[\sum^{N-1}_{k=0}(t_{k+1}-t_k)\tilde{\mathcal{L}}_{SE}(\hat{s}_{T-t_k})\leq\varepsilon_{\mathrm{score}},\]
where \begin{equation*}
    \tilde{\mathcal{L}}_{SE}
    \coloneqq\mathbb{E}_{x_{t_k}\sim \overleftarrow{\tilde{q}}_{t_k}}\left[\sum_{y\neq x_{t_k}}\left(\hat{Q}_{t_k}(x_{t_k},y)-\overleftarrow{Q}_{t_k}(x_{t_k},y)+\overleftarrow{Q}_{t_k}(x_{t_k},y)\log \frac{\overleftarrow{Q}_{t_k}(x_{t_k},y)}{\hat{Q}_{t_k}(x_{t_k},y)}\right)\right].
\end{equation*}
\end{assumption}
\Cref{assp1} measures the average estimation error of the score function along the reverse trajectory. This estimation error form is commonly used when there is no mismatch \cite{liang2025absorb,liang2025discrete,ren2025stoc-int}. Under the score-mismatched setting, the only difference is in the distribution over which the loss is evaluated: here it is evaluated over the induced target process rather than the training process. This assumption holds when the target $\tilde{q}_t$ is close to $q_t$, which is also the scenario of typical interest. For zero-shot generation discussed in \Cref{sec:zs,sec:pbset:setcons}, the scores used do not involve additional conditional training, but it is still assumed that the pretrained model with hard clamp score defined in \eqref{eq:hc} has sufficient estimation accuracy under the true conditional distribution. Note that such a mismatched loss evaluation has also been used for mismatched continuous diffusion models \cite{liang2025theory}.

\begin{assumption}\label{assp2} (Bounded Score Estimates) There exists $M>0$ such that $\forall x,y\in [S]^d$, the score estimation $\hat{s}_{T-t}(y,x)$ satisfies $M^{-1}\leq\hat{s}_{T-t}(y,x)\leq M$.
\end{assumption}

\Cref{assp2} is a typical assumption commonly used in the literature \cite{liang2025discrete,liang2025absorb,ren2025stoc-int}. It is also mild as we will later show that the dependence on $M$ is only logarithmic.

\subsection{Absorbing-Rate CTMCs with \texorpdfstring{$\tau$}{Tau}-leaping}

The following theorem provides an upper bound for the KL-divergence between the target distribution $\tilde{q}_0$ and the sampling distribution $p_T$ with early-stopping. The result is the \textit{first} convergence guarantee for the score mismatched discrete diffusion in the CTMC framework.

\begin{theorem}\label{geral:abtau}
    Under \Cref{assp1,assp2}, taking the time-steps $t_{k+1}-t_k=c\min\{1,T-t_k\}$, and with the \textbf{absorbing}-rate CTMC and the $\tau$-leaping sampler, we have
    \begin{align*}
        \mathrm{KL}(\tilde{q}_{\delta}||p_{T-\delta})& 
        \lesssim \mathcal{E}_{\mathrm{oracle}}+(1+T+\log M\delta^{-1})
        \\& 
        \qquad\cdot\sum^{N-1}_{k=0}(t_{k+1}-t_k)(T-t_k)^{-1}\mathbb{E}_{x_{t_k}\sim\overleftarrow{\tilde{q}}_{t_k}}\left[\sum_{y\neq x_{t_k}}Q(y,x_{t_k})|\Delta_{t_k}(y,x_{t_k})|\right],
    \end{align*}
    where $\mathcal{E}_{\mathrm{oracle}}$ is the sampling error from the score-matched sampling, including the initialization, discretization, and score estimation errors. To ensure a finite KL divergence, the oracle bound uses a surrogate initialization distribution \cite[Theorem~1]{liang2025absorb}.
\end{theorem}

\Cref{geral:abtau} indicates that with the absorbing-rate CTMC and $\tau$-leaping sampler,\footnote{Drawing on techniques in \cite{liang2025discrete}, our analysis can be easily adapted to analyze other samplers, including Euler method and Tweedie $\tau$-leaping.} the final KL divergence can be divided into two terms. The first term, $\mathcal{E}_{\mathrm{oracle}}$, is the intrinsic sampling error without any score mismatch. For example, under the setting and assumptions of \cite[Theorems~1~and~2]{liang2025absorb}, the absorbing-rate analysis of $\tau$-leaping gives the following upper bound of oracle error:
\[\mathcal{E}_{\mathrm{oracle}}\lesssim de^{-T}+\varepsilon_{\mathrm{score}}+d(1+T+\log M\delta^{-1})\frac{(T+\log \delta^{-1})^2}{N},\] which yields an explicit bound for \Cref{geral:abtau}. This term allows sharper bounds on the same oracle contributions without changing the score-mismatch argument as long as the same sampling setting is applied. 
The second term is an additional bias term introduced by the score mismatch at each reverse sampling step. Such a bias will gradually accumulate and finally impact the KL divergence between the target distribution and the sampling distribution. If $\Delta_t(y,x)=0$, that is, the target distribution and the training distribution are matched, the bias term will also disappear. The proof is provided in \Cref{app:geral:abtau}.

To showcase the generality of \Cref{geral:abtau}, we would later apply it to analyze the Conditional SEDD sampler in \Cref{sec:csedd} and the RHC sampler in \Cref{sec:res:setcons}.

Under additional assumptions on the score mismatch, we can further characterize the sum in \Cref{geral:abtau} and make the parameter dependence of the resulting bias explicit, as follows.

\begin{corollary}\label{corol:scmis_cons}
    Under the same setting and assumptions as in \Cref{geral:abtau}, when $|\Delta_t(y,x_t)|\lesssim1$
    , we further have 
    \begin{align*}
        \mathrm{KL}(\tilde{q}_{\delta}||p_{T-\delta})&\lesssim \mathcal{E}_{\mathrm{oracle}}+dS(1+T+\log M\delta^{-1})(T+\log \delta^{-1}).
    \end{align*}
\end{corollary}

This corollary shows that, when $|\Delta_t(y,x)|$ is uniformly upper-bounded by a constant, the bias term shows an explicit dependence on the dimension $d$ and the vocabulary size $S$. 
Such a uniformly bounded condition is rather mild. Consider a target shift example, where the ratio $\frac{\tilde{q}_0}{q_0}$ is bounded by constants that are independent of $d$, $S$, and $t$. Since the same forward process is applied to both the target and training distributions, the score mismatch admits $|\Delta_t(y,x_t)|\lesssim1$. See the detailed proof of \Cref{corol:scmis_cons} in \Cref{app:corol1}.

\subsection{Extension to Uniform-Rate CTMCs}

Using similar analysis techniques, we can also extend the result of \Cref{geral:abtau} to the uniform-rate CTMC case.
\begin{corollary}\label{geral:unitau}
    Under \Cref{assp1,assp2}, taking the time step $t_{k+1}-t_k=c\min\{1,T-t_k\}$, with the \textbf{uniform}-rate CTMC and the $\tau$-leaping sampler, we have
    \begin{align*}
        \mathrm{KL}(\tilde{q}_{\delta}||p_{T-\delta})&
        \lesssim \mathcal{E}_{\mathrm{oracle}}^{\mathrm{uni}}+S(1+\log MS\delta^{-1})\sum^{N-1}_{k=0}(t_{k+1}-t_k)\max\{1,(T-t_k)^{-1}\}
        \\&
        ~\qquad \cdot\mathbb{E}_{x_{t_k}\sim\overleftarrow{\tilde{q}}_{t_k}}\left[\sum_{y\neq x_{t_k}}Q(y,x_{t_k})|\Delta_{t_k}(y,x_{t_k})|\right],
    \end{align*}
    where $\mathcal{E}_{\mathrm{oracle}}^{\mathrm{uni}}$ is the sampling error from the score-matched sampling with uniform rate.
\end{corollary}
\Cref{geral:unitau} illustrates that, similar to the absorbing-rate CTMC case, the final KL divergence for the uniform-rate CTMC case can also be decomposed into the oracle error $\mathcal{E}_{\mathrm{oracle}}^{\mathrm{uni}}$ and an additional bias term induced by the score mismatch. This also highlights the fact that our method does not rely on a specific rate matrix but rather it reflects a general mechanism when there is score mismatch. As in \Cref{geral:abtau}, $\mathcal{E}_{\mathrm{oracle}}^{\mathrm{uni}}$ can be bounded by other applicable score-matched analyses of uniform rate and $\tau$-leaping sampler, such as \cite[Theorem~2]{liang2025discrete}. Compared to the absorbing-rate CTMC, the bias for the uniform-rate CTMC is worse by $\mathcal{O}(S)$. This arises because the uniform rate allows the token transition to other $S-1$ states and thus the upper-bound of score function under the uniform rate shows a dependence on $S$. A similar phenomenon has been observed for the oracle error in score-matched discrete diffusion models \cite{liang2025absorb,liang2025discrete,ren2025stoc-int}. The proof of \Cref{geral:unitau} is given in \Cref{app:corol2}.

\subsection{Extension to Uniformization Sampling}

After establishing the convergence results for the $\tau$-leaping sampler, we now focus on the score-mismatch analysis of the uniformization sampler. To establish the corresponding convergence guarantee, we introduce the following continuous-time integral version of the score estimation error assumption:
\begin{assumption}\label{assp3}
(Integrated Score Estimation Error)
    The $\hat{s}_{T-t}$ satisfies 
    \begin{equation*}
        \int_0^{T-\delta}\tilde{\mathcal{L}}_{SE}(\hat{s}_{T-t})\mathrm{d}t\leq\varepsilon'_{\mathrm{score}}.
    \end{equation*}
\end{assumption}
Here, we clarify the distinction between \Cref{assp3,assp1}. \Cref{assp1} indicates that the score estimation error manifests as the accumulation of the score entropy over the discrete time grid. This controls the estimation error at the selected discrete time points $t_k$, which is generally appropriate for analyzing the sampling methods based on the time discretization, for example, the $\tau$-leaping sampler \cite{campbell2022discrete,ren2025stoc-int}. In contrast, \Cref{assp3} requires controlling the score estimation error throughout the continuous time horizon, which is applicable to the analysis of the continuous-time sampling methods, for example, the uniformization sampler \cite{chen2025convergence,van1992uniformization}.
\begin{corollary}\label{geral:abuniuni}
    Under \Cref{assp2,assp3}, with \textbf{absorbing} rate and uniformization sampler, we have
    \begin{align*}
        \mathrm{KL}(\tilde{q}_{\delta}||p_{T-\delta})&\lesssim \mathcal{E}'_{\mathrm{oracle}}\\
        &~\quad+
        (1+T+\log M\delta^{-1})\int_0^{T-\delta}(T-t)^{-1}\mathbb{E}_{x_t\sim\overleftarrow{\tilde{q}}_t}\left[\sum_{y\neq x_t}Q(y,x_t)|\Delta_t|\right]\mathrm{d}t,
    \end{align*}
    where $\mathcal{E}'_{\mathrm{oracle}}$ is the sampling error from the score-matched sampling with uniformization sampler.
\end{corollary}
The uniformization sampler has been demonstrated to exactly simulate the sampling process without inducing time discretization error \cite{ren2025stoc-int}. This property is also reflected in the accumulation way of score mismatch. \Cref{geral:abuniuni} indicates that the local score mismatch in reverse dynamics will gradually accumulate in the form of a \textit{continuous-time integral} over
time and eventually impact the KL divergence between the target distribution and the generated distribution. Similarly, $\mathcal{E}'_{\mathrm{oracle}}$ can incorporate the bounds of score-matched initialization and score estimation errors. \cite[Theorem~3]{liang2025absorb} provides a concrete instance. The proof is given in \Cref{app:corol3}.

\section{Guarantees for Zero-Shot Conditional Discrete Diffusion Samplers under Fixed Observations}\label{sec:zeroshot}

In this section, to showcase the applicability of our score-mismatch framework, we utilize our general convergence analysis in \Cref{sec:gc} to study the conditional zero-shot discrete diffusion model under fixed observations introduced in \Cref{sec:zs}.
Specifically, we analyze one popular conditional zero-shot discrete diffusion sampler in the literature: the Conditional SEDD sampler \cite{lou2023discrete} (see Algorithm \ref{lousampler}).
This sampler has demonstrated empirical effectiveness in text infilling without requiring additional conditional training. However, a theoretical understanding of this sampler remains lacking, particularly regarding how hard clamping at each diffusion step affects the final sampling error.

Let $\bar\Omega\subseteq [d]$ denote the set of observed coordinates, and let $\Omega=[d]\setminus\bar\Omega$ denote the set of unobserved coordinates to be generated. The algorithm directly performs \textit{hard clamps} on the observed coordinates $\bar\Omega$ during each step of the reverse sampling process (which corresponds to a per-step hard projection in view of constrained sampling). We focus on the sampling accuracy over the unobserved coordinates $\Omega$, given a fixed non-MASK observation $y$. Specifically, for $\hat{x}$ and $x$ differing only at a coordinate $k\in\Omega$, at each time-step $t$, the hard clamp mechanism uses
\begin{equation}\label{eq:hc}
    s_{t|y}^{\text{hard-clamp}}(\hat{x}_t,x_t)=\frac{q_t(\hat{x}_t^{\Omega}|x_t^{\bar\Omega}=y)}{q_t(x_t^{\Omega}|x_t^{\bar\Omega}=y)}=\frac{q_t(\hat{x}_t^{\Omega},x_t^{\bar\Omega}=y)}{q_t(x_t^{\Omega},x_t^{\bar\Omega}=y)}.
\end{equation}
Although this method always satisfies the observation constraints, it may introduce bias on the unobserved space. This is because the hard clamp condition is enforced on the diffusion \textit{intermediate state} $x_t^{\bar\Omega}=y$, whereas the true condition is enforced on the \textit{clean data} $x_0^{\bar\Omega}=y$. Indeed, the true conditional score function at time $t$ should be 
\begin{equation}\label{def:trsc}
    s_{t|y}(\hat{x}_t,x_t)=\frac{q_t(\hat{x}_t^\Omega|x_0^{\bar\Omega}=y)}{q_t(x_t^\Omega|x_0^{\bar\Omega}=y)}.
\end{equation}
We further provide an example in \Cref{example} to illustrate the difference. From \Cref{def1}, the score mismatch for this case is given as follows.
\begin{definition}\label{def_zs}
For a fixed prompt $x_0^{\bar\Omega}=y$, let $\hat{x}\neq x$ only at index $k\in\Omega$. Given the true conditional score $s_{t|y}(\hat{x},x)$ in \eqref{def:trsc}, hard clamp score $s_{t|y}^{\text{hard-clamp}}(\hat{x},x)$ in \eqref{eq:hc}, the score mismatch $\Delta_{t|y}(\hat{x},x)$ can be defined as
    \begin{equation*}
1+\Delta_{t|y}(\hat{x},x)=\frac{s_{T-t|y}(\hat{x},x)}{s_{T-t|y}^{\text{hard-clamp}}(\hat{x},x)}.
\end{equation*}
\end{definition}
We measure the conditional sampling error through the discrepancy between the true conditional target and generated distribution on the unobserved coordinates $\Omega$. 

\subsection{Conditional SEDD under Fixed Observations with Absorbing Rate}\label{sec:csedd}

The following theorem characterizes the theoretical performance of the Conditional SEDD sampler with the absorbing-rate CTMC.
\begin{theorem}\label{lou:abtau}
   Under \Cref{assp1,assp2}, taking the time step $t_{k+1}-t_k=c\min\{1,T-t_k\}$, with the \textbf{absorbing}-rate CTMC and using $\tau$-leaping sampler, Conditional SEDD achieves 
    \begin{align*}
        \mathrm{KL}(q^\Omega_{\delta|y}||p^\Omega_{T-\delta|y})&\lesssim \mathcal{E}_{\mathrm{oracle}},
    \end{align*}
    where $q^\Omega_{\delta|y}$ denotes the target distribution on $\Omega$ given the condition $x_0^{\bar\Omega}=y$ and $p^\Omega_{T-\delta|y}$ denotes the generated distribution of the Conditional SEDD on $\Omega$. 
\end{theorem}

\Cref{lou:abtau} shows that Conditional SEDD does not introduce additional score-mismatch bias term in the convergence bound on the unobserved coordinates $\Omega$. This is due to the special structure of absorbing rate matrix. Given a fixed non-MASK observation $y$, observing $x_t^{\bar\Omega}=y$ at time $t$ implies that the initial constrained coordinates are equal to $y$ at time $0$ and remain unchanged. Therefore, the common factor regarding the transition over $\bar\Omega$ will be canceled in the score ratio, and the hard clamp score coincides with the true conditional score, as established in \eqref{app:absamescore}. Consequently, this KL upper bound only contains the oracle error $\mathcal{E}_{\mathrm{oracle}}$, which is introduced by initialization, score estimation, and time discretization. Under the stated assumptions, using unconditional score for conditional generation will achieve no extra score-mismatch bias. The proof is provided in \Cref{app:louzs}. 

\subsection{Conditional SEDD under Fixed Observations with Uniform Rate}

In this section, we provide the convergence analysis for Conditional SEDD with the uniform rate in \Cref{lou:unitau}. Unlike absorbing rate, uniform-rate CTMC assigns non-zero probabilities between the different clean tokens. Thus, conditioning $x_t^{\bar\Omega}=y$ at time $t$ does not directly imply $x_0^{\bar\Omega}=y$. The hard clamp score used by Conditional SEDD will be influenced by the contributions of different clean observations. An example is provided in \Cref{example} to illustrate this. This distinction can introduce score mismatch to the unobserved coordinates $\Omega$. Under the setting with the uniform rate, we first employ the following assumption to control this variation:
\begin{assumption}\label{assp:gamma}
    For the fixed observation $y\in [S]^{|\bar\Omega|}$, there exists a constant $0<\gamma<1$ such that 
    \begin{equation*}
        q_0(x_0^{\bar\Omega}=y|x_0^\Omega)\geq\gamma,\quad \forall x_0^\Omega\in\mathrm{supp}(q^\Omega_0)
    \end{equation*}
\end{assumption}

\Cref{assp:gamma} requires that, given any unobserved part $x_0^\Omega$, the likelihood of the fixed observation $y$ always has a positive lower bound $\gamma$. Intuitively, this assumption quantifies the compatibility between the fixed observation $y$ and all possible completions. Specifically, if there exists a completion $x_0^\Omega$ such that $q_0(x_0^{\bar\Omega}=y|x_0^\Omega)=0$, this completion has zero probability and is excluded from the target. In text generation tasks constrained by context, this assumption therefore allows the context to prefer some completions but requires other completions to remain possible. Such a uniform positive lower bound exists for a finite-state distribution with full support. For example, in structured text generation, $y$ can represent the commonly used conventional expressions such that $\gamma$ is non-negligible. With \Cref{assp:gamma}, we now provide the convergence result of Conditional SEDD with uniform-rate CTMC, as follows.
\begin{theorem}\label{lou:unitau}
   Under \Cref{assp1,assp2,assp:gamma}, taking the time step $t_{k+1}-t_k=c\min\{1,T-t_k\}$, with \textbf{uniform} rate and $\tau$-leaping sampler, Conditional SEDD achieves
    \begin{align*}
        \text{KL}(q^{\Omega}_{\delta|y}||p^{\Omega}_{T-\delta|y})&\lesssim \mathcal{E}_{\mathrm{oracle}}^{\mathrm{uni}}\\&+|\Omega|S^2(1+\log MS\delta^{-1})\left(\frac{1}{\gamma}-\gamma\right)\left(\log \frac{T+S\gamma}{1+S\gamma}+\frac{1}{S\gamma}\log\frac{\delta+S\gamma}{\delta(1+S\gamma)}\right),
    \end{align*}
    where $|\Omega|$ is the number of unobserved coordinates.
\end{theorem}

\Cref{lou:unitau} bounds the conditional sampling error in two parts. The first term is $\mathcal{E}_{\mathrm{oracle}}^{\mathrm{uni}}$, which again refers to the intrinsic error for the corresponding score-matched process. The second term refers to the additional bias induced by score mismatch, which reflects the cost of infilling. In contrast, in the absorbing rate case, Conditional SEDD achieves exact score matching on $\Omega$, as established in \Cref{lou:abtau}, whereas \Cref{lou:unitau} indicates that the uniform-rate Conditional SEDD accumulates the score mismatch over the unobserved coordinates $\Omega$. Taken together, \Cref{lou:abtau,lou:unitau} suggest that the effect of the hard clamp mechanism depends on the choice of the rate matrix employed during the diffusion processes. The proof of \Cref{lou:unitau} is given in \Cref{app:lou:unitau}.

For a fixed $\gamma$, the bias term is linearly dependent on the size of $\Omega$, which means that the conditional bias will increase by the number of positions to be completed. In contrast, when $|\Omega|$ is small, the bias will concentrate only on a small number of incomplete indices. The parameter $\gamma$ describes the compatibility of the context constraint according to the Bayes rule $q_0(x_0^\Omega|x_0^{\bar\Omega}=y)=\frac{q_0(X_0^{\bar\Omega}=y|x_0^\Omega)q_0(x^\Omega_0)}{q_0(x^{\bar\Omega}=y)}.$ When $\gamma\to 1$, the bias term will disappear, which means that the specific context is applicable for any completion, and will not provide any additional information. In contrast, as $\gamma\to 0$, the upper bound of the entire error diverges, which may occur when the context almost determines the unique missing content.

It is notable that for a fixed sequence length, increasing or decreasing the number of observed coordinates will change $|\Omega|$ while imposing different condition constraints and the corresponding compatibilities. Hence, the parameters $|\Omega|$ and $\gamma$ jointly determine the error deterioration from two distinct perspectives. $|\Omega|$ characterizes the number of unknown contents to be generated, while $\gamma$ quantifies the compatibility of the given observation with the unknown content by providing a uniform lower bound on the likelihood. Moreover, with the remaining parameters fixed, the additional score-mismatch term does not decrease as the number of reverse steps $N$ increases. The experiments in \Cref{fig:unisedd,fig:differentunisedd} further illustrate the effects of $|\Omega|$ and $\gamma$ on the sampling accuracy.

\section{Randomized Hard Clamping under Set Constraints}\label{sec:res:setcons}

Section~\ref{sec:zeroshot} characterizes the convergence error induced by the score mismatch for the Conditional SEDD sampler under a specific constraint $x^{\bar\Omega}=y$. In this section, we turn to the more general set constraints introduced in \Cref{sec:pbset:setcons}, where each constrained index is required to lie in an admissible set rather than to take a single fixed value. Such constraints support controlled text generation while maintaining the flexibility of text. Under this setting, $\bar\Omega$ denotes the set of constrained coordinates. Different from the singleton constraint, the set constraint contains multiple feasible assignment values $a\in A$. Since we can generate the remaining coordinates while hard clamping the selected assignment $a$, the additional task is to select these assignments such that the target distribution satisfies the set constraints. 

\subsection{Assignment distribution and RHC Sampler}

The key observation is that the corresponding target conditional distribution can be expressed by the mixture of the conditional distribution with distinct fixed observations. That is, for any time $t$, let $v(a):=q_0(x_0^{\bar\Omega}=a|C)$ denote the target assignment distribution over $A$, we have
\begin{equation}\label{eq:targetmix}
    q_t(x^\Omega|C)=\sum_{a\in A}v(a)q_t(x^\Omega|x_0^{\bar\Omega}=a), \quad a\in A.
\end{equation} 
Each mixture component corresponds to a fixed-observation problem. Thus, set-constrained sampling naturally involves two steps: selecting an assignment $a \sim v$ and generating the remaining coordinates given by the assignment $a$.

To construct the assignment distribution using a pretrained unconditional model, we first show that $v(a)$ admits an exact expression in terms of the unconditional score under the absorbing rate in the following proposition. 

\begin{proposition}\label{prop_distriequal}
   Let $\bar{\Omega}=\{m_1,...,m_{|\bar{\Omega}|}\},1\leq j\leq |\bar{\Omega}|$, then define an all-mask state $z_a^{(0)}=[\mathrm{MASK}]^d$ and $z_a^{(j)}$ is obtained from $z_a^{(j-1)}$ by changing the $m_j$-th coordinate from $\mathrm{MASK}$ to $a^{m_j}$. For each $a\in A$, with \textbf{absorbing} rate, and the unconditional score $s_t(y,x)$, then, for any $t>0$, we have  
   \begin{align*}
        v(a)=q_0(x_0^{\bar\Omega}=a|C)=\frac{\prod_{j=1}^{|\bar\Omega|}s_t(z_a^{(j)},z_a^{(j-1)})}{\sum_{b\in A}\prod_{j=1}^{|\bar\Omega|}s_t(z_b^{(j)},z_b^{(j-1)})},
    \end{align*}
    where $|\bar\Omega|$ refers to the number of the constrained coordinates.
\end{proposition}
Proposition \ref{prop_distriequal} reveals the main mechanism of achieving zero-shot sampling for the setting of set constraints. Interestingly, although the definition of $v$ is based on the unknown true data distribution $q_0$, $q_0(x_0^{\bar\Omega}=a|C)$ can be exactly recovered using the unconditional scores $s_t(y,x)$ under the CTMC with absorbing rate. Moreover, although the individual scores depend on the diffusion time $t$, this identity holds for any $t>0$. The proof is provided in \Cref{app:prop_distriequal}. Under the uniform rate, the reconstruction of the assignments distribution requires additional marginalization and inversion of the forward transition kernels. We discuss the computational and numerical challenges in \Cref{app:uni:rhc}.

To obtain computable assignment probabilities, for any chosen time $t$, we replace the exact score in \Cref{prop_distriequal} with the learned scores, yielding an approximation $w$ to the target assignment distribution $v$:
\begin{equation}\label{eq:construct_w}
    w(a)=\frac{\prod_{j=1}^{|\bar\Omega|}\hat{s}_t(z_a^{(j)},z_a^{(j-1)})}{\sum_{b\in A}\prod_{j=1}^{|\bar\Omega|}\hat{s}_t(z_b^{(j)},z_b^{(j-1)})}.
\end{equation}
Consequently, motivated by the mixture representation in \eqref{eq:targetmix}, we propose Randomized Hard Clamping (RHC) method shown in Algorithm \ref{RHC}. In particular, RHC first samples one constraint value $a$ from a specified assignment distribution $w(a)$, then hard clamps the constrained coordinates to $a$ throughout the reverse sampling process and generates the remaining coordinates using pretrained unconditional scores. Therefore, let $p_{T-t}^{\mathrm{HC}}(x^\Omega|x_0^{\bar\Omega}=a)$ denote the corresponding distribution on $\Omega$ for one specific value $a$ during the reverse process. We define 
\begin{equation}\label{eq:def:rhcdistri}
    p^{\text{RHC}}_{T-t}(x^\Omega|C):=\sum_{a\in A}w(a)p^{\mathrm{HC}}_{T-t}(x^\Omega|x^{\bar\Omega}_0=a),\quad \forall t\in[0,T].
\end{equation}

\begin{algorithm}[htbp]
\label{RHC}
\caption{Randomized Hard Clamping (RHC) Sampler}
\SetKwInput{Input}{Input}
\SetKwInput{Output}{Output}

\Input{Initial sample $x_0 \sim p_0$, constrained coordinates $\bar\Omega$, constraint sets $C$ on $\bar\Omega$, time grid $0=t_0<t_1<...<t_N=T-\delta$, a learned unconditional score $\hat{s}_{T-t_n}$}
Compute $w$ over $A$ using \eqref{eq:construct_w}

Sample $a\sim w$, set $x_0^{\bar\Omega}=a$

\For{$n=0$ \KwTo $N-1$}{
\For{$\hat{x}\neq x$ only at index $k\in\Omega$}{
            $s^{\mathrm{HC}}_{T-t_n|a}(\hat{x},x)=\hat{s}_{T-t_n}\left([\hat{x}^\Omega,a],[x^\Omega,a]\right)$, 
            
            where $[\cdot,a]$ refers to the whole sequence with $x^{\bar\Omega}=a$.
            $Q^{\mathrm{HC}}_{t_n|a}(x,\hat{x})=Q_{T-t_n}(\hat{x},x)s^{\mathrm{HC}}_{T-t_n|a}(\hat{x},x)$
}
    Draw $x_{t_{n+1}}$ using $Q^{\mathrm{HC}}_{t_n|a}$ (e.g., with $\tau$-leaping in \eqref{def:tau-leaping})
    
    Set $x_{t_{n+1}}^{\bar\Omega}=a$
}
\end{algorithm}

\subsection{Convergence Guarantee for RHC}

To establish the convergence guarantee of RHC, we need to control both the assignment distribution error and the sampling error for each selected assignment (see \Cref{lemma_kldecomp}). For the latter, we can reuse the convergence analysis for the preceding fixed-observation sampler for each component in \eqref{eq:def:rhcdistri}, since each assignment $a\in A$ corresponds to a reverse sampling process with hard clamping. To control the assignment distribution error introduced by constructing $w(a)$ from the estimated unconditional score, we employ the following assumption:
\begin{assumption}\label{assp:rhc}
    For any $t\in[\delta,T]$, let $w$ be the assignment distribution constructed from \eqref{eq:construct_w}. For $\forall a\in A$ and $1\leq j\leq |\bar\Omega|$, the score estimation $\hat{s}_{t}$ satisfies
    \[\frac{1}{2}(\mathbb{E}_{a\sim v}+\mathbb{E}_{a\sim w})\sum_{j=1}^{|\bar\Omega|}\left[\frac{\hat{s}_{t}(z_a^{(j)},z_a^{(j-1)})}{s_{t}(z_a^{(j)},z_a^{(j-1)})}-1-\log \frac{\hat{s}_{t}(z_a^{(j)},z_a^{(j-1)})}{s_{t}(z_a^{(j)},z_a^{(j-1)})}\right]\leq\varepsilon_{\mathrm{score}}^{\mathrm{path}}.\]
\end{assumption}
This assumption controls the accuracy of the score estimation along the paths that construct the assignment distribution $w$. Since the assignment probabilities are calculated by the product of the scores, it is natural to measure the error by the ratios $\frac{\hat{s}_t}{s_t}$. The expectations under $v$ and $w$ measure the estimation error along the construction paths under the target assignment distribution and the estimated distribution, respectively. This assumption therefore complements \Cref{assp1} regarding the required path. The averaging on both $v$ and $w$ is motivated by the symmetric KL used in the proof. Specifically, symmetric KL allows us to cancel the normalization terms in product of scores (see \eqref{eq:symkl}). Moreover, the loss achieves the optimum at the exact score, consistent with the consistency of the score entropy in \cite[Proposition~3.2]{lou2023discrete}. With exact unconditional scores, the error vanishes and the construction recovers $v=w$. 

Combining the score representation in \Cref{prop_distriequal} with the KL bound in \Cref{lemma_kldecomp}, we can now provide an upper bound on the KL divergence of the RHC sampling with the corresponding target distribution under the set constraints setting.
\begin{theorem}\label{thm_klconstraint}
   For any $t\in[\delta,T]$, let $w(a)$ be constructed by estimated unconditional scores as \eqref{eq:construct_w}. Under \Cref{assp1,assp2,assp:rhc}, taking the time step $t_{k+1}-t_k=c\min\{1,T-t_k\}$, with \textbf{absorbing} rate and $\tau$-leaping sampler, we have
    \begin{align*}
        \mathrm{KL}(q_{\delta|C}^\Omega\|p_{T-\delta|C}^{\Omega,\mathrm{RHC}})&\lesssim \mathcal{E}_{\mathrm{oracle}} + \sqrt{|\bar\Omega|\varepsilon_{\mathrm{score}}^{\mathrm{path}}}
        ,
    \end{align*}
    where $\mathcal{E}_{\mathrm{oracle}}$ is a common upper bound on the initialization, score estimation, and discretization errors for the fixed-observation sampling for each $a\in A$ and $|\bar\Omega|$ refers to the number of the constrained coordinates.
\end{theorem}
Theorem \ref{thm_klconstraint} indicates that extending a fixed observation to set constraints will not alter the primary form of the convergence results. The additional term is $\sqrt{|\bar\Omega|\varepsilon_{\mathrm{score}}^{\mathrm{path}}}$, which is induced by approximating the distribution $v$ by $w$ based on the estimated score $\hat{s}_t$. Hence, we focus on the dependence on the score estimation error $\varepsilon_{\mathrm{score}}^{\mathrm{path}}$. If the score is further trained, the additional bound will decrease and vanish as $\varepsilon_{\mathrm{score}}^{\mathrm{path}}\to0$. In particular, with exact unconditional scores, the construction yields $w=v$, and the assignment distribution error vanishes. Moreover, with other parameters fixed, this additional cost exhibits a square-root dependence on the number of constrained coordinates $|\bar\Omega|$ and does not introduce an explicit dependence on the size of the feasible sequence set $|A|=\prod_{m\in\bar\Omega}|A_m|$, although the number of feasible sequences may grow combinatorially with the number of constrained values. The proof is given in \Cref{app:thm_klconstraint}.

\section{Experiment Results}\label{main:exp}
In this section, we provide experiment results to validate the results obtained in this paper.

\subsection{Conditional SEDD under Fixed Observation}\label{exp:seddnum}

In our simulation, we conduct numerical experiments on a finite data distribution $q_0$ under a fixed observation $x_0^{\bar\Omega}=y$. The data distribution is constructed using normalized exponential weights, resulting in a non-uniform distribution with coordinate dependencies. Upon choosing a data distribution, we estimate the corresponding KL divergence between the target distribution and the sampling distribution from the samples. More details about our settings can be found in \Cref{app:exp:consedd}. 

To verify our theory, we first compare Conditional SEDD in \cite{lou2023discrete} with the oracle sampler under absorbing and uniform rates. The oracle sampler uses the true conditional score $s_{t|y}$ of the corresponding target distribution. \Cref{fig:absedd,fig:unisedd} show the comparison of the KL divergence of Conditional SEDD and the oracle sampler. Next, to further explore the effects of $|\Omega|$ and $\gamma$ for uniform-rate Conditional SEDD, we compare the performance of Conditional SEDD with different numbers of free coordinates. The results are reported in \Cref{fig:differentunisedd}.

\begin{figure}[htbp]
    \centering
    \includegraphics[width=0.5\linewidth]{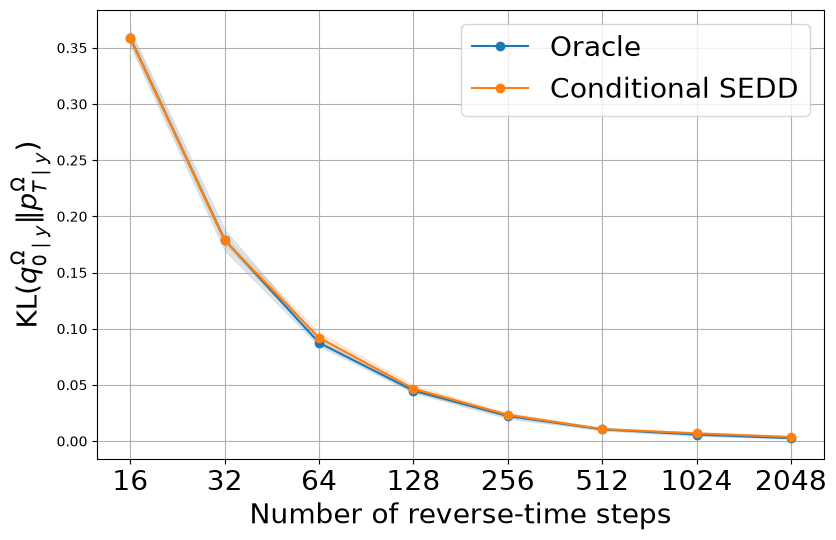}
    \caption{Comparison of the Conditional SEDD and Oracle with fixed observation under absorbing rate in terms of the final KL divergence between the target and sampling conditional distributions.}
    \label{fig:absedd}
\end{figure}

\Cref{fig:absedd} shows that the Conditional SEDD is closely aligned with the oracle sampler under the absorbing rate. Both errors decrease as the number of reverse time steps increases. This observation is consistent with \Cref{lou:abtau}, where the hard-clamped score coincides with the true conditional score on $\Omega$. Thus, Conditional SEDD introduces no additional structural score-mismatch bias for the generated results.

\begin{figure}[htbp]
    \centering
    \includegraphics[width=0.5\linewidth]{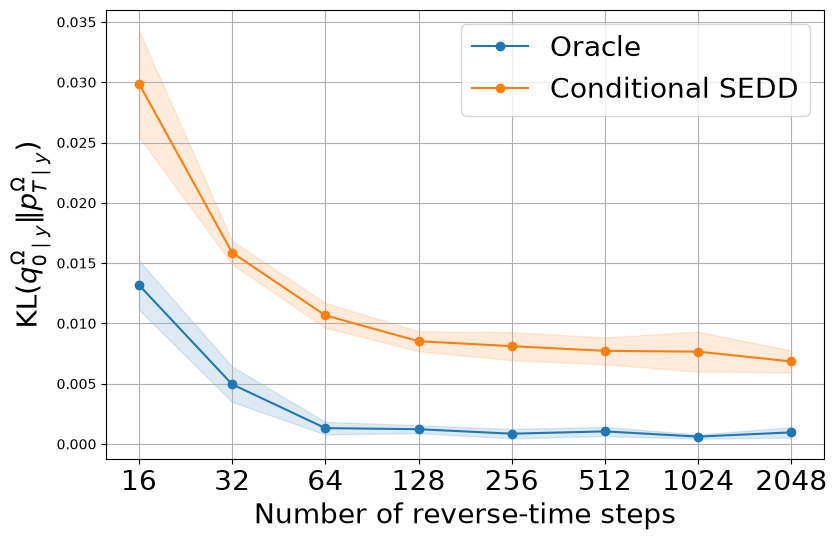}
    \caption{Comparison of the Conditional SEDD and Oracle with fixed observation under uniform rate in terms of the final KL divergence between the target and sampling conditional distributions.}
    \label{fig:unisedd}
\end{figure}

Under the uniform rate, \Cref{fig:unisedd} shows a different phenomenon. Although the KL divergences of both samplers gradually decrease, the Conditional SEDD is stuck on an irreducible plateau for a large number of steps, and the oracle sampler achieves a smaller KL divergence. This observation is consistent with the score mismatch issue in \Cref{lou:unitau}, where Conditional SEDD introduces an additional bias term. Thus, these results illustrate that the mismatch between Conditional SEDD and the oracle sampler will not be eliminated even with a finer time discretization.

\begin{figure}[htbp]
    \centering
    \includegraphics[width=0.5\linewidth]{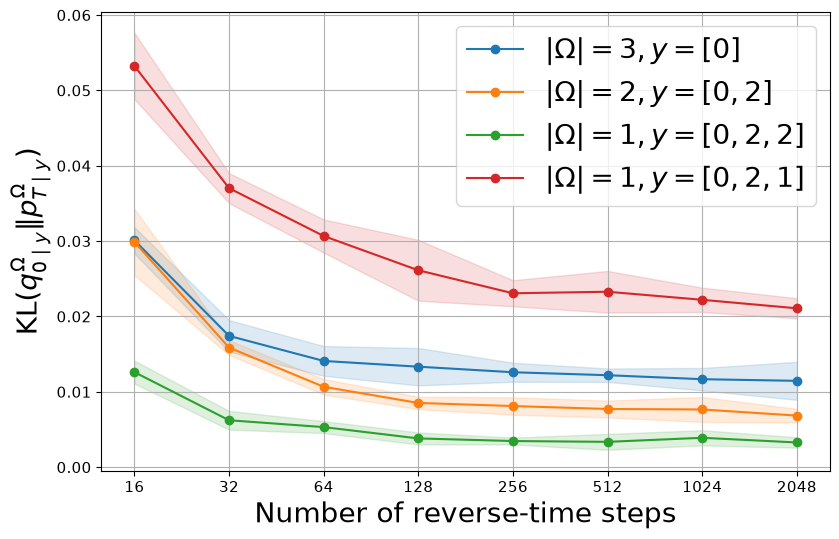}
    \caption{Comparison of Conditional SEDD with different numbers of free coordinates under the uniform rate in terms of the final KL divergence between the target and sampling conditional distributions. The different sizes of $\Omega$ are $1$, $2$, $3$. 
    }
    \label{fig:differentunisedd}
\end{figure}

\Cref{fig:differentunisedd} illustrates the trends of the KL divergence of uniform-rate Conditional SEDD with different fixed observations and numbers of unobserved coordinates. For $y=[0]$, $[0,2]$, and $[0,2,2]$, the KL divergence increases with the larger size of the unobserved set. However, with $y=[0,2,1]$, the KL divergence is the worst among the four settings. The different trends across these settings do not contradict the theoretical bound in \Cref{lou:unitau}, since the upper bound depends on both the number of free coordinates $|\Omega|$ and the condition compatibility parameter $\gamma$ in \Cref{assp:gamma}. As the condition changes from $y=[0,2,2]$ to $[0,2,1]$, we observe a significant increase in KL divergence, which is consistent with \Cref{lou:unitau} that a smaller $\gamma$ will deteriorate the sampling error.

\subsection{RHC under Set Constraints}\label{exp:rhc}
\paragraph{Synthetic Experiments}
For the synthetic experiments, we use the same distribution above and impose the set constraints. We compare our RHC in Algorithm~\ref{RHC} with the oracle sampler under absorbing-rate CTMC. The assignment probabilities $w(a)$ are computed from the exact unconditional scores. 

\begin{figure}[htbp]
    \centering
    \includegraphics[width=0.5\linewidth]{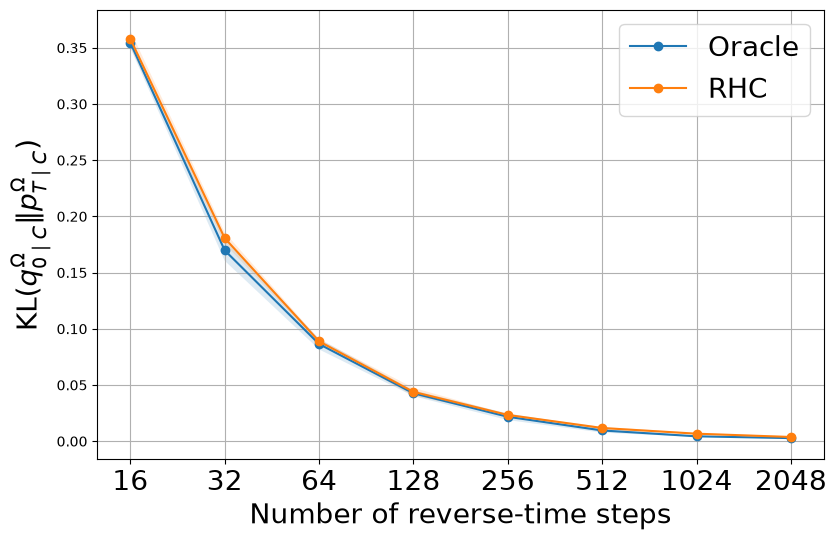}
    \caption{Comparison of the RHC and Oracle in terms of the final KL divergence between the target and sampling conditional distributions.}
    \label{fig:rhcklresult}
\end{figure}

\Cref{fig:rhcklresult} shows the comparison of the KL divergence of the oracle sampler and RHC. We evaluate the KL divergence on the generated coordinates without early stopping. In this experiment, since the assignment distribution can be calculated from the data distribution, there is no approximation error in the assignment distribution. In particular, $w=v$ and $\varepsilon_{\mathrm{score}}^{\mathrm{path}}=0$. Therefore, the bound in \Cref{thm_klconstraint} reduces to $\mathcal{E}_{\mathrm{oracle}}$. As shown in \Cref{fig:rhcklresult}, the KL divergence of both methods decreases as the number of steps increases, and RHC achieves sampling errors comparable to those of the oracle sampler. This result demonstrates the performance of RHC with the ideal constructions of the assignment distribution and the exact unconditional scores, corroborating the results in \Cref{thm_klconstraint}. Moreover, hard clamping in RHC introduces no score mismatch for each assignment in this experiment with the absorbing rate.

\paragraph{Text Experiments}
Beyond synthetic data, we also conduct a series of text infilling experiments with coordinatewise set constraints to validate the effectiveness of our RHC sampler. In our experiments, several positions in the sequence are selected as the constrained coordinates, where multiple candidate tokens are included (see the detailed list in \Cref{app:realsetting}). The sampler is required to select a feasible assignment for these positions in the different scenarios and generate the remaining tokens. 

We compare RHC with a uniform-assignment baseline. Different from RHC that selects assignments using the assignment distributions calculated from the trained unconditional scores, the baseline method samples the feasible assignments independently and uniformly. Detailed settings are shown in \Cref{app:exp:rhc}. We report the generative perplexity, which is widely used to evaluate the quality of the generated text \cite{lou2023discrete,schiff2025simple,ou2024your,sahoo2024simple}, of the proposed RHC sampler and the baseline sampler in \Cref{exp:result}.

\begin{table}[htbp]
 \caption{Summary of perplexity in transportation, study, game, cooking, project and weather. The results include the mean $\pm$ standard deviation and total generation time of five independent trials.}
 \label{exp:result}
 \centering
 \begin{tabular}{lccc}
   \toprule
   Sampler     & Transportation  & Study  & Game\\
   \midrule
   Uniform Assignment & 100.61$\pm$1.72 (1610 s) & 99.54$\pm$1.44 (1611 s)& 103.35$\pm$1.52 (1605 s)\\
   RHC     & 61.01$\pm$1.53 (1660 s) & 62.82$\pm$1.08 (1664 s)& 69.57$\pm$0.75 (1659 s)\\
   \midrule
   Sampler     & Cooking & Project & Weather \\
   \midrule
   Uniform Assignment & 135.37$\pm$0.76 (1601 s) & 98.26$\pm$2.66 (1606 s)& 91.22$\pm$1.31 (1600 s)\\
   RHC     & 116.82$\pm$2.06 (1655 s) & 77.48$\pm$1.35 (1660 s) & 65.20$\pm$1.27 (1650 s)\\
   \bottomrule
 \end{tabular}
\end{table}

\Cref{exp:result} summarizes the results of the perplexity in the six different constraint scenarios. The construction of the assignment distribution takes $50\sim54$ seconds per scenario, and the distribution can be reused among independent trials for each scenario. Compared with the uniform-assignment baseline, RHC consistently achieves lower perplexity across all evaluated scenarios. Meanwhile, the total generation time for RHC increases only by $3.1\%\sim3.4\%$. Hence, the experiments demonstrate that our RHC sampler can effectively satisfy the desired set constraints without compromising the fluency of conditional generation. \Cref{app:realresults1,app:realresults2} provide several generated results for reference.

\section{Conclusion}

In this paper, we have provided the convergence guarantees for the general score-mismatched discrete diffusion model. Based on the general results, we have further studied the zero-shot conditional generation. Our analysis characterizes the bias of Conditional SEDD under the fixed observations. We have shown that the hard clamp score coincides with the true conditional score under the absorbing rate case but not under the uniform rate case. We have derived an additional score-mismatched bias term in the KL upper bound under the uniform rate. To further extend conditional sampling to coordinate-wise set constraints, we have designed a novel Randomized Hard Clamping sampler with an absorbing rate. A natural direction for future work is to extend the current analysis and samplers to more general multimodal data and more complex constraint information.


\bibliographystyle{IEEEtran}
\bibliography{diffusion}

\clearpage
\appendices
\begin{center}
    \huge \textbf{Appendix}
\end{center}


\allowdisplaybreaks
\startcontents[section]
{
\hypersetup{linkcolor=blue}
\printcontents[section]{l}{1}{\setcounter{tocdepth}{2}}
}

\crefalias{section}{appendix} 
\crefalias{subsection}{appendix}

\section{Related Works}\label{app:relaconti}

\textbf{Conditional Continuous Diffusion Model}

There exists sufficient research that focuses on the conditional generation related to the continuous diffusion model. These studies are primarily divided into two fields. One field focuses mainly on the guidance mechanism. For example, \cite{dhariwal2021diffusion} proposed classifier guidance, which utilized an external classifier to modify the score function to inject conditional information into the diffusion model. \cite{ho2022classifier} introduced classifier-free guidance, which combined an unconditional diffusion and a conditional diffusion to achieve conditional generation and eliminate the classifier. \cite{radford2021learning} introduced the CLIP model to measure the similarity between images and text, which can therefore be used as external guidance in diffusion models. \cite{GLIDE2022} further studied the guidance mechanism in text-to-image generation and compared CLIP guidance and classifier-free guidance. \cite{Avrahami2022Blend} modified the score function for the image editing area through the CLIP guidance and used the unconditional diffusion for the non-editing area.

Another field is based on the unconditional diffusion without extra training. SDEdit \cite{meng2022sdedit} realized the image synthesis and editing by noising and denoising the user input. This method was trained on unlabeled data without retraining on the new input or tasks. RePaint \cite{lugmayr2022repaint} fixed the noised pixels to constrain the sampling process and achieve conditional generation. \cite{rout2023theoretical} theoretically analyzed the RePaint \cite{lugmayr2022repaint} and proposed RePaint+ to correct the bias induced by RePaint \cite{lugmayr2022repaint}. \cite{Choi2021ILVR} proposed the ILVR method to implement image generation given reference images. This method refined the latent variables in the unconditional process to approximate the conditional process. DDRM in \cite{kawar2022denoising} solved the linear inverse problem by the posterior sampling based on a pretrained diffusion model, and further DPS in \cite{chung2023diffusion} extended the posterior sampling method to the general inverse problem. DDNM and DDNM+ \cite{wang2023zeroshot} are also well-known zero-shot frameworks that only rely on the pretrained diffusion model. For the theoretical works, \cite{liang2025theory} provided the theoretical guarantee of the zero-shot conditional diffusion model by introducing the score mismatch concept under the continuous state cases.

In addition, \cite{saharia2022palette,zhang2023adding,brooks2023instructpix2pix} designed the specialized conditional frameworks in the diffusion model to implement the conditional generation rather than to solely control the inference process.

\section{Proof of \texorpdfstring{\Cref{geral:abtau}}{Theorem III.3}}\label{app:geral:abtau}

First, we decompose the KL divergence according to \cite[Theorem~1]{liang2025discrete}, as 
\begin{align*}
    &\mathrm{KL}(\overleftarrow{\tilde{q}}_{t_N}\|p_{t_N})\\
    &\leq\mathrm{KL}(\overleftarrow{\tilde{q}}_0\|p_0)\\
    &\qquad+\sum^{N-1}_{k=0}\int^{t_{k+1}}_{t_k}\mathbb{E}_{x_t\sim\overleftarrow{\tilde{q}}_t}\left[\sum_{y\neq x_t}\left(\hat{Q}_t(x_t,y)-\overleftarrow{\tilde{Q}}_t(x_t,y)+\overleftarrow{\tilde{Q}}_t(x_t,y)\log \frac{\overleftarrow{\tilde{Q}}_t(x_t,y)}{\hat{Q}_t(x_t,y)}\right)\right]\mathrm{d}t.
\end{align*}
We define $\tilde{g}_t(x_t):=\sum_{y\neq x_t}\hat{Q}_t(x_t,y)-\overleftarrow{\tilde{Q}}_t(x_t,y)+\overleftarrow{\tilde{Q}}_t(x_t,y)\log \frac{\overleftarrow{\tilde{Q}}_t(x_t,y)}{\hat{Q}_t(x_t,y)}$ and we have $\hat{Q}_t(x_{t_k},y)=\hat{Q}_{t_k}(x_{t_k},y)$ for $\tau$-leaping. Then, the KL divergence can be rewritten as
\begin{align}\label{eq:kldecomp}
    &\text{KL}(\overleftarrow{\tilde{q}}_T\|p_T)\nonumber\\
    &\leq\underbrace{\text{KL}(\overleftarrow{\tilde{q}}_0\|p_0)}_{\text{Initialization error}}+\underbrace{\sum^{N-1}_{k=0}(t_{k+1}-t_k)\mathbb{E}_{x_{t_k}\sim\overleftarrow{\tilde{q}}_{t_k}}[\tilde{g}_{t_k}(x_{t_k})]}_{\text{Estimation error}}\nonumber\\
    &+\underbrace{\sum^{N-1}_{k=0}\int^{t_{k+1}}_{t_k}\left(\mathbb{E}_{x_t\sim\overleftarrow{\tilde{q}}_t,x_{t_k}\sim\overleftarrow{\tilde{q}}_{t_k}}\left[\tilde{g}_t(x_t)-\tilde{g}_t(x_{t_k})\right]+\mathbb{E}_{x_{t_k}\sim\overleftarrow{\tilde{q}}_{t_k}}\left[\tilde{g}_t(x_{t_k})-\tilde{g}_{t_k}(x_{t_k})\right]\right)\text{d}t}_{\text{Discretization error}}.
\end{align}
Since $\tilde{q}_t$ evolves under the same absorbing rate generator, the initialization error and the discretization error can be controlled by the existing convergence analyses for $\tau$-leaping (cf. \cite[Theorem~2]{liang2025absorb}). These two errors, together with the score estimation error $\varepsilon_{\mathrm{score}}$ controlled by \Cref{assp1}, constitute the oracle error $\mathcal{E}_{\mathrm{oracle}}$ as follows:
\begin{align*}
    \mathcal{E}_{\mathrm{oracle}}:=\mathcal{E}_{\mathrm{init}}+\mathcal{E}_{\mathrm{disc}}+\varepsilon_{\mathrm{score}},
\end{align*}
where $\mathcal{E}_{\mathrm{init}}$ and $\mathcal{E}_{\mathrm{disc}}$ refer to the upper bounds of the initialization and discretization errors, respectively. 

Substituting these contributions into \eqref{eq:kldecomp} gives
\begin{align}\label{eq:step2bias}
        \text{KL}(\tilde{q}_{\delta}\|p_{T-\delta})\lesssim& \mathcal{E}_{\mathrm{init}}+\mathcal{E}_{\mathrm{disc}}+\sum^{N-1}_{k=0}(t_{k+1}-t_k)\mathbb{E}_{x_{t_k}\sim\overleftarrow{\tilde{q}}_{t_k}}[\tilde{g}_{t_k}(x_{t_k})]\nonumber\\
        =& \mathcal{E}_{\mathrm{oracle}}+\sum^{N-1}_{k=0}(t_{k+1}-t_k)\mathbb{E}_{x_{t_k}\sim\overleftarrow{\tilde{q}}_{t_k}}[\tilde{g}_{t_k}(x_{t_k})]-\varepsilon_{\mathrm{score}}.
\end{align}
This representation isolates the oracle sampling error. Subtracting $\varepsilon_{\mathrm{score}}$ is because the score estimation error is already included in $\mathcal{E}_{\mathrm{oracle}}$. It remains to quantify the contribution of the score mismatch by bounding the difference between the second term and the score estimation error controlled by \Cref{assp1}.

We now focus on $\sum^{N-1}_{k=0}(t_{k+1}-t_k)\mathbb{E}_{x_{t_k}\sim\overleftarrow{\tilde{q}}_{t_k}}[\tilde{g}_{t_k}(x_{t_k})]$. Here we note a useful relationship given \Cref{def1}:
\[ \overleftarrow{\tilde{Q}}_{t_k}(x_{t_k},y) = \overleftarrow{Q}_{t_k}(x_{t_k},y)(1+\Delta_{t_k}(y,x_{t_k})). \]
The substitution of the target distribution will result in the perturbation of the score function, consequently manifested as a multiplicative correction to the original rate matrix by the score mismatch.
For notational simplicity, we write $\Delta_{t}(y,x)$ as $\Delta_{t}$. Thus,
\begin{align*}
    \tilde{g}_{t_k}(x_{t_k})
    &=\sum_{y\neq x_{t_k}}\left(\hat{Q}_{t_k}(x_{t_k},y)-\overleftarrow{\tilde{Q}}_{t_k}(x_{t_k},y)+\overleftarrow{\tilde{Q}}_{t_k}(x_{t_k},y)\log \frac{\overleftarrow{\tilde{Q}}_{t_k}(x_{t_k},y)}{\hat{Q}_{t_k}(x_{t_k},y)}\right)\\
    &=\sum_{y\neq x_{t_k}}\Bigg[\hat{Q}_{t_k}(x_{t_k},y)-\overleftarrow{Q}_{t_k}(x_{t_k},y)(1+\Delta_{t_k})\\
    &\qquad+\overleftarrow{Q}_{t_k}(x_{t_k},y)(1+\Delta_{t_k})\log \frac{\overleftarrow{Q}_{t_k}(x_{t_k},y)(1+\Delta_{t_k})}{\hat{Q}_{t_k}(x_{t_k},y)}\Bigg]\\
    &=\sum_{y\neq x_{t_k}}\Bigg[\hat{Q}_{t_k}(x_{t_k},y)-\overleftarrow{Q}_{t_k}(x_{t_k},y)-\overleftarrow{Q}_{t_k}(x_{t_k},y)\Delta_{t_k}\\
    &\qquad+\overleftarrow{Q}_{t_k}(x_{t_k},y)(1+\Delta_{t_k})\left(\log \frac{\overleftarrow{Q}_{t_k}(x_{t_k},y)}{\hat{Q}_{t_k}(x_{t_k},y)}+
    \log(1+\Delta_{t_k})\right)\Bigg]\\
    &=\sum_{y\neq x_{t_k}}\Bigg[\hat{Q}_{t_k}(x_{t_k},y)-\overleftarrow{Q}_{t_k}(x_{t_k},y)+\overleftarrow{Q}_{t_k}(x_{t_k},y)\log \frac{\overleftarrow{Q}_{t_k}(x_{t_k},y)}{\hat{Q}_{t_k}(x_{t_k},y)}\\
    &\qquad +\overleftarrow{Q}_{t_k}(x_{t_k},y)((1+\Delta_{t_k})
    \log(1+\Delta_{t_k})-\Delta_{t_k})\\
    &\qquad +\overleftarrow{Q}_{t_k}(x_{t_k},y)\Delta_{t_k}\log \frac{\overleftarrow{Q}_{t_k}(x_{t_k},y)}{\hat{Q}_{t_k}(x_{t_k},y)}\Bigg]\\
    &\stackrel{(i)}{\leq}\sum_{y\neq x_{t_k}}\Bigg[\hat{Q}_{t_k}(x_{t_k},y)-\overleftarrow{Q}_{t_k}(x_{t_k},y)+\overleftarrow{Q}_{t_k}(x_{t_k},y)\log \frac{\overleftarrow{Q}_{t_k}(x_{t_k},y)}{\hat{Q}_{t_k}(x_{t_k},y)}\\
    &\qquad +\overleftarrow{Q}_{t_k}(x_{t_k},y)\Delta_{t_k}^2\\
    &\qquad +\overleftarrow{Q}_{t_k}(x_{t_k},y)\Delta_{t_k}\log \frac{\overleftarrow{Q}_{t_k}(x_{t_k},y)}{\hat{Q}_{t_k}(x_{t_k},y)}\Bigg]\\
    &=\sum_{y\neq x_{t_k}}\Bigg[\hat{Q}_{t_k}(x_{t_k},y)-\overleftarrow{Q}_{t_k}(x_{t_k},y)+\overleftarrow{Q}_{t_k}(x_{t_k},y)\log \frac{\overleftarrow{Q}_{t_k}(x_{t_k},y)}{\hat{Q}_{t_k}(x_{t_k},y)}\\
    &\qquad +Q(y,x_{t_k})(\tilde{s}_{T-t_k}(y,x_{t_k})-s_{T-t_k}(y,x_{t_k}))\Delta_{t_k}\\
    &\qquad +Q(y,x_{t_k})\Delta_{t_k}s_{T-t_k}(y,x_{t_k})\log \frac{s_{T-t_k}(y,x_{t_k})}{\hat{s}_{T-t_k}(y,x_{t_k})}\Bigg]\\
    &\stackrel{(ii)}{\leq} \sum_{y\neq x_{t_k}}\Bigg[\hat{Q}_{t_k}(x_{t_k},y)-\overleftarrow{Q}_{t_k}(x_{t_k},y)+\overleftarrow{Q}_{t_k}(x_{t_k},y)\log \frac{\overleftarrow{Q}_{t_k}(x_{t_k},y)}{\hat{Q}_{t_k}(x_{t_k},y)}\\
    &\qquad +Q(y,x_{t_k})|\Delta_{t_k}|(T-t_k)^{-1}\\
    &\qquad +Q(y,x_{t_k})|\Delta_{t_k}|s_{T-t_k}(y,x_{t_k})\left|\log \frac{s_{T-t_k}(y,x_{t_k})}{\hat{s}_{T-t_k}(y,x_{t_k})}\right|\Bigg]\\
    &\stackrel{(iii)}{\lesssim} \sum_{y\neq x_{t_k}}\Bigg[\hat{Q}_{t_k}(x_{t_k},y)-\overleftarrow{Q}_{t_k}(x_{t_k},y)+\overleftarrow{Q}_{t_k}(x_{t_k},y)\log \frac{\overleftarrow{Q}_{t_k}(x_{t_k},y)}{\hat{Q}_{t_k}(x_{t_k},y)}\\
    &\qquad +Q(y,x_{t_k})|\Delta_{t_k}|(T-t_k)^{-1}(1+T+\log M\delta^{-1})\Bigg],
\end{align*}
where $(i)$ holds since $\forall x>-1$, we have $\log(1+x)\leq x$ and hence $(1+x)\log(1+x)-x\leq x^2$. $(ii)$ is because of the following. Since from \cite[Lemma~1]{liang2025absorb}, we have a general bound for all score functions under the absorbing rate when $Q(y,x)>0$ and $q_t(y)>0$, which is $0\leq s_{T-t_k}(y,x_{t_k}),\tilde{s}_{T-t_k}(y,x_{t_k})\leq (T-t_k)^{-1}$, then $|\tilde{s}_{T-t_k}(y,x_{t_k})-s_{T-t_k}(y,x_{t_k})|\leq (T-t_k)^{-1}$. $(iii)$ holds since 
\begin{align*}
    &s_{T-t_k}(y,x_{t_k})\left|\log \frac{s_{T-t_k}(y,x_{t_k})}{\hat{s}_{T-t_k}(y,x_{t_k})}\right|\\
    =&(T-t_k)^{-1}s_{T-t_k}(y,x_{t_k})(T-t_k)\left|\log (s_{T-t_k}(y,x_{t_k})(T-t_k))-\log (T-t_k)-\log\hat{s}_{T-t_k}(y,x_{t_k})\right|\\
    \lesssim& (T-t_k)^{-1}(1+T+\log M\delta^{-1}),
\end{align*}
where $s_{T-t_k}(y,x_{t_k})(T-t_k)\left|\log (s_{T-t_k}(y,x_{t_k})(T-t_k))\right|\lesssim 1$ for $0\leq s_{T-t_k}(y,x_{t_k})(T-t_k)\leq1$. Hence, taking the expectation yields
\begin{align}\label{eq:bias}
&\mathbb{E}_{x_{t_k}\sim\overleftarrow{\tilde{q}}_{t_k}}[\tilde{g}_{t_k}(x_{t_k})] \nonumber \\
\lesssim& \mathbb{E}_{x_{t_k}\sim\overleftarrow{\tilde{q}}_{t_k}}\left[\sum_{y\neq x_{t_k}}\Bigg(\hat{Q}_{t_k}(x_{t_k},y)-\overleftarrow{Q}_{t_k}(x_{t_k},y)
+\overleftarrow{Q}_{t_k}(x_{t_k},y)\log \frac{\overleftarrow{Q}_{t_k}(x_{t_k},y)}{\hat{Q}_{t_k}(x_{t_k},y)}\Bigg)\right]\nonumber\\
&\qquad +(T-t_k)^{-1}(1+T+\log M\delta^{-1})\mathbb{E}_{x_{t_k}\sim\overleftarrow{\tilde{q}}_{t_k}}\left[\sum_{y\neq x_{t_k}}Q(y,x_{t_k})|\Delta_{t_k}|\right].
\end{align}

The first term is controlled directly by \Cref{assp1}:
\begin{align}
    \label{eq:est}
    \sum^{N-1}_{k=0}&(t_{k+1}-t_k)\cdot\nonumber\\
    &\mathbb{E}_{x_{t_k}\sim\overleftarrow{\tilde{q}}_{t_k}}\left[\sum_{y\neq x_{t_k}}\left(\hat{Q}_{t_k}(x_{t_k},y)-\overleftarrow{Q}_{t_k}(x_{t_k},y)+\overleftarrow{Q}_{t_k}(x_{t_k},y)\log \frac{\overleftarrow{Q}_{t_k}(x_{t_k},y)}{\hat{Q}_{t_k}(x_{t_k},y)}\right)\right]\leq \varepsilon_{\text{score}}.
\end{align}

Finally, combining the preceding difference $\sum^{N-1}_{k=0}(t_{k+1}-t_k)\mathbb{E}_{x_{t_k}\sim\overleftarrow{\tilde{q}}_{t_k}}[\tilde{g}_{t_k}(x_{t_k})]-\varepsilon_{\mathrm{score}}$ with \eqref{eq:bias} and \eqref{eq:est}, we obtain 
\begin{align}\label{eq:indepbias}
    &\sum^{N-1}_{k=0}(t_{k+1}-t_k)\mathbb{E}_{x_{t_k}\sim\overleftarrow{\tilde{q}}_{t_k}}[\tilde{g}_{t_k}(x_{t_k})]-\varepsilon_{\mathrm{score}}\nonumber\\
    \lesssim&(1+T+\log M\delta^{-1})\sum^{N-1}_{k=0}(t_{k+1}-t_k)(T-t_k)^{-1}\mathbb{E}_{x_{t_k}\sim\overleftarrow{\tilde{q}}_{t_k}}\left[\sum_{y\neq x_{t_k}}Q(y,x_{t_k})|\Delta_{t_k}|\right].
\end{align}

Substituting \eqref{eq:indepbias} into \eqref{eq:step2bias} yields 
\begin{align*}
        \text{KL}(\tilde{q}_{\delta}\|p_{T-\delta})&\lesssim \mathcal{E}_{\mathrm{oracle}}+(1+T+\log M\delta^{-1})\\&\qquad\qquad\qquad\cdot\sum^{N-1}_{k=0}(t_{k+1}-t_k)(T-t_k)^{-1}\mathbb{E}_{x_{t_k}\sim\overleftarrow{\tilde{q}}_{t_k}}\left[\sum_{y\neq x_{t_k}}Q(y,x_{t_k})|\Delta_{t_k}|\right],
\end{align*}
This completes the proof.

\section{Proof of \texorpdfstring{\Cref{corol:scmis_cons}}{Corollary III.4}}\label{app:corol1}

Starting from \Cref{def1}, for all $x_t\neq y$, we have
\begin{align}\label{corl1:scms}
        \Delta_t(y,x_t)&=\frac{\tilde{s}_{T-t}(y,x_t)}{s_{T-t}(y,x_t)} - 1 \nonumber\\
        &=\left(\frac{\tilde{q}_{T-t}(y)}{\tilde{q}_{T-t}(x_t)}\right)\bigg/\left(\frac{q_{T-t}(y)}{q_{T-t}(x_t)}\right)-1\nonumber\\
        &=\left(\frac{\tilde{q}_{T-t}(y)}{q_{T-t}(y)}\right)\bigg/\left(\frac{\tilde{q}_{T-t}(x_t)}{q_{T-t}(x_t)}\right)-1.
\end{align}
Suppose that there exist two universal constants $C_1$ and $C_2$ satisfying $0<C_1\leq C_2<\infty$, such that $C_1\leq\frac{\tilde{q}_0(x_0)}{q_0(x_0)}\leq C_2$ for all $x_0$, then we can have
\begin{equation*}
    \tilde{q}_t(x)=\sum_{x_0}\tilde{q}_{t|0}(x|x_0)\tilde{q}_0(x_0)=\sum_{x_0}\tilde{q}_{t|0}(x|x_0)q_0(x_0)\frac{\tilde{q}_0(x_0)}{q_0(x_0)}.
\end{equation*}
Hence, 
\begin{equation*}
    C_1\sum_{x_0}\tilde{q}_{t|0}(x|x_0)q_0(x_0)\leq\tilde{q}_t(x)\leq C_2\sum_{x_0}\tilde{q}_{t|0}(x|x_0)q_0(x_0).
\end{equation*}
Since $q_0$ and $\tilde{q}_0$ share the same forward kernel, then
\begin{equation*}
    C_1q_t(x)\leq\tilde{q}_t(x)\leq C_2q_t(x).
\end{equation*}
Based on this, we can obtain
\begin{equation}\label{app:ratioineq}
    C_1\leq\frac{\tilde{q}_t(x)}{q_t(x)}\leq C_2.
\end{equation}
Plugging \eqref{app:ratioineq} back to \eqref{corl1:scms}, we have
\begin{equation*}
    |\Delta_t|\leq \frac{C_2}{C_1}-1,
\end{equation*}
which implies that $|\Delta_t|\lesssim 1$. Thus, we further have
\begin{align*}
        &(1+T+\log M\delta^{-1})\sum^{N-1}_{k=0}(t_{k+1}-t_k)(T-t_k)^{-1}\mathbb{E}_{x_{t_k}\sim\overleftarrow{\tilde{q}}_{t_k}}\left[\sum_{y\neq x_{t_k}}Q(y,x_{t_k})|\Delta_{t_k}|\right]\\\stackrel{(i)}{\lesssim}
        &d(S-1)(1+T+\log M\delta^{-1})\sum^{N-1}_{k=0}(t_{k+1}-t_k)(T-t_k)^{-1}
        \\\stackrel{(ii)}{\lesssim}
        &dS(1+T+\log M\delta^{-1})(T+\log\delta^{-1}),
\end{align*}
where $(i)$ holds because we have $\sum_{y\neq x}Q(y,x)=m(x)(S-1)\leq d(S-1)$ with the absorbing rate and $m(x)$ is the number of tokens at MASK state in $x$. $(ii)$ holds since $\sum^{N-1}_{k=0}(t_{k+1}-t_k)(T-t_k)^{-1}\leq T+\log\delta^{-1}$ with early stopping. Now, the proof is complete.

\section{Proof of \texorpdfstring{\Cref{geral:unitau}}{Corollary III.5}}\label{app:corol2}

With most of the proof similar to \Cref{geral:abtau}, we use the corresponding score-matched convergence result for uniform-rate $\tau$-leaping to control $\mathcal{E}_{\mathrm{oracle}}^{\mathrm{uni}}$, thus it remains to replace the applicable bounds of the score function and rate matrix in the bias term \eqref{eq:step2bias} by the corresponding counterparts. 

Starting from $\tilde{g}_{t_k}(x_{t_k})$, we can bound it as follows:
\begin{align*}
    \tilde{g}_{t_k}(x_{t_k})
    &\leq\sum_{y\neq x_{t_k}}\Bigg[\hat{Q}_{t_k}(x_{t_k},y)-\overleftarrow{Q}_{t_k}(x_{t_k},y)+\overleftarrow{Q}_{t_k}(x_{t_k},y)\log \frac{\overleftarrow{Q}_{t_k}(x_{t_k},y)}{\hat{Q}_{t_k}(x_{t_k},y)}\\
    &\qquad +Q(y,x_{t_k})(\tilde{s}_{T-t_k}(y,x_{t_k})-s_{T-t_k}(y,x_{t_k}))\Delta_{t_k}\\
    &\qquad +Q(y,x_{t_k})\Delta_{t_k}s_{T-t_k}(y,x_{t_k})\log \frac{s_{T-t_k}(y,x_{t_k})}{\hat{s}_{T-t_k}(y,x_{t_k})}\Bigg]\\
    &\stackrel{(i)}{\lesssim} \sum_{y\neq x_{t_k}}\Bigg[\hat{Q}_{t_k}(x_{t_k},y)-\overleftarrow{Q}_{t_k}(x_{t_k},y)+\overleftarrow{Q}_{t_k}(x_{t_k},y)\log \frac{\overleftarrow{Q}_{t_k}(x_{t_k},y)}{\hat{Q}_{t_k}(x_{t_k},y)}\\
    &\qquad +Q(y,x_{t_k})|\Delta_{t_k}|S\max \{1,(T-t_k)^{-1}\}\left(1+\left|\log \frac{s_{T-t_k}(y,x_{t_k})}{\hat{s}_{T-t_k}(y,x_{t_k})}\right|\right)\Bigg]\\
    &\stackrel{(ii)}{\lesssim} \sum_{y\neq x_{t_k}}\Bigg[\hat{Q}_{t_k}(x_{t_k},y)-\overleftarrow{Q}_{t_k}(x_{t_k},y)+\overleftarrow{Q}_{t_k}(x_{t_k},y)\log \frac{\overleftarrow{Q}_{t_k}(x_{t_k},y)}{\hat{Q}_{t_k}(x_{t_k},y)}\\
    &\qquad +Q(y,x_{t_k})|\Delta_{t_k}|S\max\{1,(T-t_k)^{-1}\}(1+\log MS\delta^{-1})\Bigg],
\end{align*}
where $(i)$ follows from
$\frac{1}{S\max \{1,(T-t_k)^{-1}\}}\lesssim s_{T-t_k}(y,x_{t_k}),\tilde{s}_{T-t_k}(y,x_{t_k})\lesssim S\max \{1,(T-t_k)^{-1}\}$ given in \cite[Lemma~2]{liang2025discrete}. $(ii)$ holds since $\left|\log \frac{s_{T-t_k}(y,x_{t_k})}{\hat{s}_{T-t_k}(y,x_{t_k})}\right|\lesssim \log S+\log \delta^{-1}+\log M$ with $M^{-1}\leq\hat{s}_{T-t_k}(y,x_{t_k})\leq M$ from \Cref{assp2}.

Hence, with \Cref{assp1}, taking the expectation and summing over the time steps, we can derive the bias term as follows:
\begin{align*}
&\sum^{N-1}_{k=0}(t_{k+1}-t_k)\mathbb{E}_{x_{t_k}\sim\overleftarrow{\tilde{q}}_{t_k}}[\tilde{g}_{t_k}(x_{t_k})]-\varepsilon_{\mathrm{score}}\\
\lesssim& S(1+\log MS\delta^{-1})\sum^{N-1}_{k=0}(t_{k+1}-t_k)\max\{1,(T-t_k)^{-1}\}\mathbb{E}_{x_{t_k}\sim\overleftarrow{\tilde{q}}_{t_k}}\left[\sum_{y\neq x_{t_k}}Q(y,x_{t_k})|\Delta_{t_k}|\right].
\end{align*}
Now, the proof is complete.

\section{Proof of \texorpdfstring{\Cref{geral:abuniuni}}{Corollary III.7}}\label{app:corol3}

The uniformization sampler has been proved to achieve no discretization error \cite{chen2025convergence}. According to \cite[Corollary~1]{liang2025discrete} and \cite[Theorem~3]{liang2025absorb}, the KL divergence can be decomposed into
\begin{align*}
    \mathrm{KL}(\overleftarrow{\tilde{q}}_{T-\delta}\|p_{T-\delta})&\leq \mathrm{KL}(\overleftarrow{\tilde{q}}_0\|p_0)\\&+\mathbb{E}_{x_{0:T-\delta}\sim\overleftarrow{\tilde{q}}_{0:T-\delta}}\left[\int_0^{T-\delta}\sum_{y\neq x_t}\right.\\
    &\qquad\left.\left(\hat{Q}_t(x_t,y)-\overleftarrow{\tilde{Q}}_t(x_t,y)+\overleftarrow{\tilde{Q}}_t(x_t,y)\log\frac{\overleftarrow{\tilde{Q}}_t(x_t,y)}{\hat{Q}_t(x_t,y)}\right)\text{d}t\right].
\end{align*}

Also, we define $\tilde{g}_t(x_t)=\sum_{y\neq x_t}\hat{Q}_t(x_t,y)-\overleftarrow{\tilde{Q}}_t(x_t,y)+\overleftarrow{\tilde{Q}}_t(x_t,y)\log \frac{\overleftarrow{\tilde{Q}}_t(x_t,y)}{\hat{Q}_t(x_t,y)}$
and similarly have 
\begin{align*}
    \tilde{g}_t(x_t)&\lesssim \sum_{y\neq x_t}\Bigg[\hat{Q}_t(x_t,y)-\overleftarrow{Q}_t(x_t,y)+\overleftarrow{Q}_t(x_t,y)\log \frac{\overleftarrow{Q}_t(x_t,y)}{\hat{Q}_t(x_t,y)}\\
    &+Q(y,x_t)|\Delta_t|(T-t)^{-1}(1+T+\log M\delta^{-1})\Bigg].
\end{align*}

Given \Cref{assp3}, the estimation error can be upper-bounded by
\begin{align}\label{eq:uniest}
\mathbb{E}_{x_{0:T-\delta}\sim\overleftarrow{\tilde{q}}_{0:T-\delta}}\left[\int_0^{T-\delta}\sum_{y\neq x_t}\left(\hat{Q}_t(x_t,y)-\overleftarrow{Q}_t(x_t,y)+\overleftarrow{Q}_t(x_t,y)\log \frac{\overleftarrow{Q}_t(x_t,y)}{\hat{Q}_t(x_t,y)}\right)\text{d}t\right]\lesssim \varepsilon'_{\text{score}}.
\end{align}

Next, we consider the bias term. According to Fubini's theorem, we can exchange the integral and expectation, then under the absorbing rate,
\begin{align*}
&\int_0^{T-\delta}\mathbb{E}_{x_t\sim\overleftarrow{\tilde{q}}_t}\left[\sum_{y\neq x_t}\left(Q(y,x_t)|\Delta_t|(T-t)^{-1}(1+T+\log M\delta^{-1})\right)\right]\text{d}t\\
=&(1+T+\log M\delta^{-1})\int_0^{T-\delta}(T-t)^{-1}\mathbb{E}_{x_t\sim\overleftarrow{\tilde{q}}_t}\left[\sum_{y\neq x_t}Q(y,x_t)|\Delta_t|\right]\text{d}t.
\end{align*}
Now, combining the initialization error and the estimation error in \eqref{eq:uniest} as $\mathcal{E}'_{\mathrm{oracle}}$, we complete the proof of \Cref{geral:abuniuni}.

\section{Proof of \texorpdfstring{\Cref{lou:abtau}}{Theorem IV.1}}\label{app:louzs}

We first begin with the true conditional score. Given $x_0^{\bar\Omega}=y$, we have
\begin{align}
    q_t(x_t^\Omega|x_0^{\bar\Omega}=y)&=\sum_{x_0^{\Omega}}q_t(x_t^\Omega,x_0^{\Omega}|x_0^{\bar\Omega}=y)\nonumber\\
    &=\sum_{x_0^{\Omega}}q_{t|0}(x^{\Omega}_t|x_0^{\Omega},x_0^{\bar\Omega}=y)q_0(x_0^{\Omega}|x^{\bar\Omega}_0=y)\nonumber\\
    &=\sum_{x_0^{\Omega}}q_{t|0}(x^{\Omega}_t|x_0^{\Omega})q_0(x_0^{\Omega}|x^{\bar\Omega}_0=y)\nonumber\\
    &=\sum_{x_0^{\Omega}}\prod _{i\in \Omega}q_{t|0}(x_t^i|x_0^i)q_0(x_0^{\Omega}|x^{\bar\Omega}_0=y).\label{eq:truecon}
\end{align}

Recall that $\Omega$ refers to the set of unobserved indices that need to be filled and $\bar\Omega$ refers to the observed indices. Given $\hat{x}_t\neq x_t$ only at index $k$, we only need to discuss the case where $k\in\Omega$.

If $k\in \Omega$, with \eqref{eq:truecon}, the true conditional score can be expressed as
\begin{align}\label{eq:truesc}
    s_{T-t|y}(\hat{x},x)=\frac{\sum_{x_0^{\Omega}}q_{T-t|0}(\hat{x}^k|x_0^k)\prod _{i\in {\Omega \backslash \{k\}}}q_{T-t|0}(x^i|x_0^i)q_0(x_0^{\Omega}|x^{\bar\Omega}_0=y)}{\sum_{x_0^{\Omega}}q_{T-t|0}(x^k|x_0^k)\prod _{i\in {\Omega \backslash \{k\}}}q_{T-t|0}(x^i|x_0^i)q_0(x_0^{\Omega}|x^{\bar\Omega}_0=y)}.
\end{align}
Recall that in \eqref{eq:hc}, since we have 
$q_{T-t}(x^{\Omega}|x^{\bar\Omega}=y)=\frac{q_{T-t}(x^{\Omega},x^{\bar\Omega}=y)}{q_{T-t}(x^{\bar\Omega}=y)}$, the hard clamp score can be rewritten as $\frac{q_{T-t}(\hat{x}^{\Omega},x^{\bar\Omega}=y)}{q_{T-t}(x^{\Omega},x^{\bar\Omega}=y)}$ by Bayes' rule. Also, from \Cref{def_zs}, the score mismatch is
\begin{equation*}
    1+\Delta_{t|y}=\frac{q_{T-t}(\hat{x}^{\Omega}|x^{\bar\Omega}_0=y)}{q_{T-t}(x^{\Omega}|x^{\bar\Omega}_0=y)}\bigg/\frac{q_{T-t}(\hat{x}^{\Omega},x^{\bar\Omega}=y)}{q_{T-t}(x^{\Omega},x^{\bar\Omega}=y)}.
\end{equation*}
Here, for simplicity of notation, we write $\Delta_{t|y}(\hat{x},x)$ as $\Delta_{t|y}$ in the following. To bound the above score mismatch, we first rewrite the joint distribution, 
\begin{align*}
    &q_{T-t}(x^{\Omega},x^{\bar\Omega}=y)\\=&\sum_{x_0} q_{T-t|0}(x^{\Omega},x^{\bar\Omega}=y|x_0)q_0(x_0)\\=&\sum_{x_0} \prod_{i\in \Omega}q_{T-t|0}(x^i|x_0^i)\prod_{j\in \bar\Omega}q_{T-t|0}(y^j|x_0^j) q_0(x_0)\\=&\sum_{x_0^\Omega}\sum_{x_0^{\bar\Omega}} \prod_{i\in \Omega}q_{T-t|0}(x^i|x_0^i)\prod_{j\in \bar\Omega}q_{T-t|0}(y^j|x_0^j) q_0(x_0^\Omega,x_0^{\bar\Omega})\\\stackrel{(i)}{=}&\sum_{x_0^\Omega} \prod_{i\in \Omega}q_{T-t|0}(x^i|x_0^i)\prod_{j\in \bar\Omega}q_{T-t|0}(y^j|y^j) q_0(x_0^\Omega,x_0^{\bar\Omega}=y)\\=&\sum_{x_0^\Omega} \prod_{i\in \Omega}q_{T-t|0}(x^i|x_0^i)\prod_{j\in \bar\Omega}q_{T-t|0}(y^j|y^j) q_0(x_0^\Omega|x_0^{\bar\Omega}=y)q_0(x_0^{\bar\Omega}=y)\\=&\prod_{j\in \bar\Omega}q_{T-t|0}(y^j|y^j)q_0(x_0^{\bar\Omega}=y)\sum_{x_0^\Omega} \prod_{i\in \Omega}q_{T-t|0}(x^i|x_0^i)q_0(x_0^\Omega|x_0^{\bar\Omega}=y).
\end{align*}
Here, we explain $(i)$ as follows: in a forward CTMC with absorbing rate, the state transitions are accepted only from non-MASK to MASK. Since there is no mask in the data and the state $x^{\bar\Omega}_{T-t} = y$ contains no mask tokens, the state $x_0^{\bar\Omega}$ must also be $y$.

Then, the hard clamp score in \eqref{eq:hc} will be 
\begin{align}\label{app:absamescore}
s_{T-t|y}^{\text{hard-clamp}}(\hat{x},x)=&\frac{q_{T-t}(\hat{x}^{\Omega}|x^{\bar\Omega}=y)}{q_{T-t}(x^{\Omega}|x^{\bar\Omega}=y)}\nonumber
    \\=&\frac{q_{T-t}(\hat{x}^{\Omega},x^{\bar\Omega}=y)}{q_{T-t}(x^{\Omega},x^{\bar\Omega}=y)}
    \nonumber\\=&\frac{\prod_{j\in \bar\Omega}q_{T-t|0}(y^j|y^j)q_0(x_0^{\bar\Omega}=y)\sum_{x_0^\Omega} \prod_{i\in \Omega}q_{T-t|0}(\hat{x}^i|x_0^i)q_0(x_0^\Omega|x_0^{\bar\Omega}=y)}{\prod_{j\in \bar\Omega}q_{T-t|0}(y^j|y^j)q_0(x_0^{\bar\Omega}=y)\sum_{x_0^\Omega} \prod_{i\in \Omega}q_{T-t|0}(x^i|x_0^i)q_0(x_0^\Omega|x_0^{\bar\Omega}=y)}\nonumber
    \\=&\frac{\sum_{x_0^\Omega} q_{T-t|0}(\hat{x}^k|x_0^k)\prod_{i\in {\Omega \backslash \{k\}}}q_{T-t|0}(x^i|x_0^i)q_0(x_0^\Omega|x_0^{\bar\Omega}=y)}{\sum_{x_0^\Omega} q_{T-t|0}(x^k|x_0^k)\prod_{i\in {\Omega \backslash \{k\}}}q_{T-t|0}(x^i|x_0^i)q_0(x_0^\Omega|x_0^{\bar\Omega}=y)}
    \\\stackrel{(i)}{=}&s_{T-t|y}(\hat{x},x).\nonumber
\end{align}
$(i)$ holds since we have already derived the true score in \eqref{eq:truesc}.
The hard clamp score is identical to the true conditional score, and thus we conclude that for $k\in\Omega$,
\begin{equation}\label{eq:nobias}
    \Delta_{t|y}=0.
\end{equation}
Applying \Cref{geral:abtau} yields the following claimed bound:
\begin{align*}
        \mathrm{KL}(q^\Omega_{\delta|y}\|p^\Omega_{T-\delta|y})&\lesssim \mathcal{E}_{\mathrm{oracle}}.
    \end{align*}

\section{Proof of \texorpdfstring{\Cref{lou:unitau}}{Theorem IV.3}}\label{app:lou:unitau}

In this section, we present the proof of \Cref{lou:unitau}. 
If $k\in\Omega$, the hard clamp score with the uniform rate is
\begin{align*}
    &s_{T-t|y}^{\text{hard-clamp}}(\hat{x},x)\\=
    &\frac{q_{T-t}(\hat{x}^{\Omega}|x^{\bar\Omega}=y)}{q_{T-t}(x^{\Omega}|x^{\bar\Omega}=y)}\\=
    &\frac{\sum_{x_0^\Omega}\sum_{x_0^{\bar\Omega}} q_{T-t|0}(\hat{x}^k|x_0^k)\prod_{i\in \Omega\setminus\{k\}}q_{T-t|0}(x^i|x_0^i)\prod_{j\in \bar\Omega}q_{T-t|0}(y^j|x_0^j) q_0(x_0^\Omega,x_0^{\bar\Omega})}{\sum_{x_0^\Omega}\sum_{x_0^{\bar\Omega}} q_{T-t|0}(x^k|x_0^k)\prod_{i\in \Omega\setminus\{k\}}q_{T-t|0}(x^i|x_0^i)\prod_{j\in \bar\Omega}q_{T-t|0}(y^j|x_0^j) q_0(x_0^\Omega,x_0^{\bar\Omega})}\\=
    &\frac{\sum_{x_0^\Omega}q_{T-t|0}(\hat{x}^k|x_0^k)\prod_{i\in \Omega\setminus\{k\}}q_{T-t|0}(x^i|x_0^i)\sum_{x_0^{\bar\Omega}}\prod_{j\in \bar\Omega}q_{T-t|0}(y^j|x_0^j) q_0(x_0^\Omega,x_0^{\bar\Omega})}{\sum_{x_0^\Omega}q_{T-t|0}(x^k|x_0^k)\prod_{i\in \Omega\setminus\{k\}}q_{T-t|0}(x^i|x_0^i)\sum_{x_0^{\bar\Omega}}\prod_{j\in \bar\Omega}q_{T-t|0}(y^j|x_0^j) q_0(x_0^\Omega,x_0^{\bar\Omega})}\\\stackrel{(i)}{=}
    &\frac{(\frac{1}{S}-\frac{1}{S}e^{-(T-t)})\sum_{x_0^\Omega:x_0^k\neq\hat{x}^k}\pi(x_0^\Omega)+(\frac{1}{S}+\frac{S-1}{S}e^{-(T-t)})\sum_{x_0^\Omega:x_0^k=\hat{x}^k}\pi(x_0^\Omega)}{(\frac{1}{S}+\frac{S-1}{S}e^{-(T-t)})\sum_{x_0^\Omega:x_0^k=x^k}\pi(x_0^\Omega)+(\frac{1}{S}-\frac{1}{S}e^{-(T-t)})\sum_{x_0^\Omega:x_0^k\neq x^k}\pi(x_0^\Omega)}\\=
    &\frac{(\frac{1}{S}-\frac{1}{S}e^{-(T-t)})\sum_{x_0^\Omega}\pi(x_0^\Omega)+e^{-(T-t)}\sum_{x_0^\Omega:x_0^k=\hat{x}^k}\pi(x_0^\Omega)}{(\frac{1}{S}-\frac{1}{S}e^{-(T-t)})\sum_{x_0^\Omega}\pi(x_0^\Omega)+e^{-(T-t)}\sum_{x_0^\Omega:x_0^k=x^k}\pi(x_0^\Omega)}\\=
    &\frac{(\frac{1}{S}-\frac{1}{S}e^{-(T-t)})+e^{-(T-t)}\frac{\sum_{x_0^\Omega:x_0^k=\hat{x}^k}\pi(x_0^\Omega)}{\sum_{x_0^\Omega}\pi(x_0^\Omega)}}{(\frac{1}{S}-\frac{1}{S}e^{-(T-t)})+e^{-(T-t)}\frac{\sum_{x_0^\Omega:x_0^k=x^k}\pi(x_0^\Omega)}{\sum_{x_0^\Omega}\pi(x_0^\Omega)}},
 \end{align*}
 where $\pi(x_0^\Omega):=\prod_{i\in \Omega\setminus\{k\}}q_{T-t|0}(x^i|x_0^i)\sum_{x_0^{\bar\Omega}}\prod_{j\in \bar\Omega}q_{T-t|0}(y^j|x_0^j) q_0(x_0^\Omega,x_0^{\bar\Omega})$, and $(i)$ holds because the transition kernel is given by $q_{t|0}(x^i|x_0^i)=(1-e^{-t})/S$ when $x^i\neq x_0^i$; $(1+(S-1)e^{-t})/S$, otherwise.

With \eqref{eq:truesc}, the true conditional score is
 \begin{align*}
    &s_{T-t|y}(\hat{x},x)\\=&\frac{\sum_{x_0^{\Omega}}q_{T-t|0}(\hat{x}^k|x_0^k)\prod _{i\in {\Omega \backslash \{k\}}}q_{T-t|0}(x^i|x_0^i)q_0(x_0^{\Omega}|x^{\bar\Omega}_0=y)}{\sum_{x_0^{\Omega}}q_{T-t|0}(x^k|x_0^k)\prod _{i\in {\Omega \backslash \{k\}}}q_{T-t|0}(x^i|x_0^i)q_0(x_0^{\Omega}|x^{\bar\Omega}_0=y)}\\=&\frac{\sum_{x_0^\Omega:x_0^k\neq\hat{x}^k}(\frac{1}{S}-\frac{1}{S}e^{-(T-t)})\pi_{\text{true}}(x_0^\Omega)+\sum_{x_0^\Omega:x_0^k=\hat{x}^k}(\frac{1}{S}+\frac{S-1}{S}e^{-(T-t)})\pi_{\text{true}}(x_0^\Omega)}{\sum_{x_0^\Omega:x_0^k=x^k}(\frac{1}{S}+\frac{S-1}{S}e^{-(T-t)})\pi_{\text{true}}(x_0^\Omega)+\sum_{x_0^\Omega:x_0^k\neq x^k}(\frac{1}{S}-\frac{1}{S}e^{-(T-t)})\pi_{\text{true}}(x_0^\Omega)}\\=&\frac{(\frac{1}{S}-\frac{1}{S}e^{-(T-t)})\sum_{x_0^\Omega:x_0^k\neq\hat{x}^k}\pi_{\text{true}}(x_0^\Omega)+(\frac{1}{S}+\frac{S-1}{S}e^{-(T-t)})\sum_{x_0^\Omega:x_0^k=\hat{x}^k}\pi_{\text{true}}(x_0^\Omega)}{(\frac{1}{S}+\frac{S-1}{S}e^{-(T-t)})\sum_{x_0^\Omega:x_0^k=x^k}\pi_{\text{true}}(x_0^\Omega)+(\frac{1}{S}-\frac{1}{S}e^{-(T-t)})\sum_{x_0^\Omega:x_0^k\neq x^k}\pi_{\text{true}}(x_0^\Omega)}\\=&\frac{(\frac{1}{S}-\frac{1}{S}e^{-(T-t)})\sum_{x_0^\Omega}\pi_{\text{true}}(x_0^\Omega)+e^{-(T-t)}\sum_{x_0^\Omega:x_0^k=\hat{x}^k}\pi_{\text{true}}(x_0^\Omega)}{(\frac{1}{S}-\frac{1}{S}e^{-(T-t)})\sum_{x_0^\Omega}\pi_{\text{true}}(x_0^\Omega)+e^{-(T-t)}\sum_{x_0^\Omega:x_0^k=x^k}\pi_{\text{true}}(x_0^\Omega)}\\=&\frac{(\frac{1}{S}-\frac{1}{S}e^{-(T-t)})+e^{-(T-t)}\frac{\sum_{x_0^\Omega:x_0^k=\hat{x}^k}\pi_{\text{true}}(x_0^\Omega)}{\sum_{x_0^\Omega}\pi_{\text{true}}(x_0^\Omega)}}{(\frac{1}{S}-\frac{1}{S}e^{-(T-t)})+e^{-(T-t)}\frac{\sum_{x_0^\Omega:x_0^k=x^k}\pi_{\text{true}}(x_0^\Omega)}{\sum_{x_0^\Omega}\pi_{\text{true}}(x_0^\Omega)}}
\end{align*}
where we define $\pi_{\text{true}}(x_0^\Omega):=\prod _{i\in {\Omega \backslash \{k\}}}q_{T-t|0}(x^i|x_0^i)q_0(x_0^{\Omega},x^{\bar\Omega}_0=y)$.

The above two expressions show that the hard clamp score and the true conditional score can be reformulated in the same form using $\pi(x_0^{\Omega})$ and $\pi_{\mathrm{true}}(x_0^{\Omega})$, respectively. To upper bound the score mismatch, we first establish a comparison between $\pi(x_0^{\Omega})$ and $\pi_{\mathrm{true}}(x_0^{\Omega})$ in the following lemma.
\begin{lemma}\label{lem:pi}
    With \Cref{assp:gamma}, we have
     \begin{equation*}
     \left(\frac{1}{S}+\frac{S-1}{S}e^{-(T-t)}\right)^{|\bar{\Omega}|}\pi_{\mathrm{true}}(x_0^{\Omega})\leq\pi(x_0^{\Omega})\leq \left(\frac{1}{S}+\frac{S-1}{S}e^{-(T-t)}\right)^{|\bar{\Omega}|}\frac{1}{\gamma}\pi_{\mathrm{true}}(x_0^{\Omega}).
 \end{equation*}
\end{lemma}
\begin{proof}
    See \Cref{app:proofpi}.
\end{proof}

Thus, with \Cref{lem:pi}, for $x_0^\Omega$ with $x_0^k=x^k$, we can derive
 \begin{align}\label{eq:pixk}
     \left(\frac{1}{S}+\frac{S-1}{S}e^{-(T-t)}\right)^{|\bar{\Omega}|}\sum_{x_0^\Omega:x_0^k=x^k}\pi_{\text{true}}(x_0^{\Omega})&\leq\sum_{x_0^\Omega:x_0^k=x^k}\pi(x_0^{\Omega})\nonumber\\
     &\leq \left(\frac{1}{S}+\frac{S-1}{S}e^{-(T-t)}\right)^{|\bar{\Omega}|}\frac{1}{\gamma}\sum_{x_0^\Omega:x_0^k=x^k}\pi_{\text{true}}(x_0^{\Omega}).
 \end{align}
Similarly, for all $x_0^\Omega$, we can derive
  \begin{align}\label{eq:pix0}
     \left(\frac{1}{S}+\frac{S-1}{S}e^{-(T-t)}\right)^{|\bar{\Omega}|}\sum_{x_0^\Omega}\pi_{\text{true}}(x_0^{\Omega})&\leq\sum_{x_0^\Omega}\pi(x_0^{\Omega})\nonumber\\
     &\qquad\leq \left(\frac{1}{S}+\frac{S-1}{S}e^{-(T-t)}\right)^{|\bar{\Omega}|}\frac{1}{\gamma}\sum_{x_0^\Omega}\pi_{\text{true}}(x_0^{\Omega}).
 \end{align}
 Taking the ratio of \eqref{eq:pixk} and \eqref{eq:pix0}, we have
 \begin{align}\label{eq:finalrela}
\gamma\frac{\sum_{x_0^\Omega:x_0^k=x^k}\pi_{\text{true}}(x_0^\Omega)}{\sum_{x_0^\Omega}\pi_{\text{true}}(x_0^\Omega)}\leq \frac{\sum_{x_0^\Omega:x_0^k=x^k}\pi(x_0^\Omega)}{\sum_{x_0^\Omega}\pi(x_0^\Omega)}\leq \frac{1}{\gamma}\frac{\sum_{x_0^\Omega:x_0^k=x^k}\pi_{\text{true}}(x_0^\Omega)}{\sum_{x_0^\Omega}\pi_{\text{true}}(x_0^\Omega)}.
 \end{align}

Define \[\hat{f}_{\text{true}}:=\frac{\sum_{x_0^\Omega:x_0^k=\hat{x}^k}\pi_{\text{true}}(x_0^\Omega)}{\sum_{x_0^\Omega}\pi_{\text{true}}(x_0^\Omega)}, \quad f_{\text{true}}:=\frac{\sum_{x_0^\Omega:x_0^k=x^k}\pi_{\text{true}}(x_0^\Omega)}{\sum_{x_0^\Omega}\pi_{\text{true}}(x_0^\Omega)},\] and \[\hat{f}:=\frac{\sum_{x_0^\Omega:x_0^k=\hat{x}^k}\pi(x_0^\Omega)}{\sum_{x_0^\Omega}\pi(x_0^\Omega)}, \quad f:=\frac{\sum_{x_0^\Omega:x_0^k=x^k}\pi(x_0^\Omega)}{\sum_{x_0^\Omega}\pi(x_0^\Omega)}.\]
Then, with \Cref{def_zs} and \eqref{eq:finalrela}, we have 
\begin{align*}
    1+\Delta_{t|y}
    &=\frac{s_{T-t|y}(\hat{x},x)}{s_{T-t|y}^{\text{hard-clamp}}(\hat{x},x)}
    \\
    &=\frac{(\frac{1}{S}-\frac{1}{S}e^{-(T-t)})+e^{-(T-t)}\hat{f}_{\text{true}}}{(\frac{1}{S}-\frac{1}{S}e^{-(T-t)})+e^{-(T-t)}f_{\text{true}}}\cdot\frac{(\frac{1}{S}-\frac{1}{S}e^{-(T-t)})+e^{-(T-t)}f}{(\frac{1}{S}-\frac{1}{S}e^{-(T-t)})+e^{-(T-t)}\hat{f}}\\
    &\leq\frac{(\frac{1}{S}-\frac{1}{S}e^{-(T-t)})+e^{-(T-t)}\hat{f}_{\text{true}}}{(\frac{1}{S}-\frac{1}{S}e^{-(T-t)})+e^{-(T-t)}f_{\text{true}}}\cdot\frac{(\frac{1}{S}-\frac{1}{S}e^{-(T-t)})+e^{-(T-t)}\frac{1}{\gamma}f_{\text{true}}}{(\frac{1}{S}-\frac{1}{S}e^{-(T-t)})+e^{-(T-t)}\gamma\hat{f}_{\text{true}}}\\
    &=\frac{(\frac{1}{S}-\frac{1}{S}e^{-(T-t)})+e^{-(T-t)}\hat{f}_{\text{true}}}{(\frac{1}{S}-\frac{1}{S}e^{-(T-t)})+e^{-(T-t)}\gamma\hat{f}_{\text{true}}}\cdot\frac{(\frac{1}{S}-\frac{1}{S}e^{-(T-t)})+e^{-(T-t)}\frac{1}{\gamma}f_{\text{true}}}{(\frac{1}{S}-\frac{1}{S}e^{-(T-t)})+e^{-(T-t)}f_{\text{true}}}\\
    &\stackrel{(i)}{\leq}\frac{(\frac{1}{S}-\frac{1}{S}e^{-(T-t)})+e^{-(T-t)}}{(\frac{1}{S}-\frac{1}{S}e^{-(T-t)})+e^{-(T-t)}\gamma}\cdot\frac{(\frac{1}{S}-\frac{1}{S}e^{-(T-t)})+e^{-(T-t)}\frac{1}{\gamma}}{(\frac{1}{S}-\frac{1}{S}e^{-(T-t)})+e^{-(T-t)}}\\
    &=\frac{(\frac{1}{S}-\frac{1}{S}e^{-(T-t)})+e^{-(T-t)}\frac{1}{\gamma}}{(\frac{1}{S}-\frac{1}{S}e^{-(T-t)})+e^{-(T-t)}\gamma}\\
    &=\frac{1+\frac{S}{e^{T-t}-1}\frac{1}{\gamma}}{1+\frac{S}{e^{T-t}-1}\gamma},
\end{align*}
where $(i)$ holds since $0\leq\hat{f}_{\text{true}},f_{\text{true}}\leq1$ and $G(u)=\frac{(\frac{1}{S}-\frac{1}{S}e^{-(T-t)})+e^{-(T-t)}u}{(\frac{1}{S}-\frac{1}{S}e^{-(T-t)})+e^{-(T-t)}\gamma u}$, $H(v)=\frac{(\frac{1}{S}-\frac{1}{S}e^{-(T-t)})+e^{-(T-t)}\frac{1}{\gamma}v}{(\frac{1}{S}-\frac{1}{S}e^{-(T-t)})+e^{-(T-t)}v}$ are monotonically increasing on $[0,1]$.

Applying \Cref{geral:unitau}, we obtain
\begin{align}\label{eq:app:uni:bar}
&S(1+\log MS\delta^{-1})\sum^{N-1}_{k=0}(t_{k+1}-t_k)\max\{1,(T-t_k)^{-1}\}\mathbb{E}_{x_{t_k}\sim\overleftarrow{q}_{t_k|y}^{\Omega}}\left[\sum_{y\neq x_{t_k}}Q(y,x_{t_k})|\Delta_{t_k|y}|\right]\nonumber\\
\stackrel{(i)}{\lesssim}
&dS(1+\log MS\delta^{-1})\sum^{N-1}_{k=0}(t_{k+1}-t_k)\max\{1,(T-t_k)^{-1}\}\frac{|\Omega|}{d}\left(\frac{1+\frac{S}{e^{T-t_k}-1}\frac{1}{\gamma}}{1+\frac{S}{e^{T-t_k}-1}\gamma}-1\right)\nonumber\\
=
&|\Omega|S(1+\log MS\delta^{-1})\sum^{N-1}_{k=0}(t_{k+1}-t_k)\max\{1,(T-t_k)^{-1}\}\frac{\frac{S}{e^{T-t_k}-1}(\frac{1}{\gamma}-\gamma)}{1+\frac{S}{e^{T-t_k}-1}\gamma}\nonumber\\
\leq
&|\Omega|S(1+\log MS\delta^{-1})\sum^{N-1}_{k=0}(t_{k+1}-t_k)\max\{1,(T-t_k)^{-1}\}\frac{\frac{S}{T-t_k}(\frac{1}{\gamma}-\gamma)}{1+\frac{S}{T-t_k}\gamma}\nonumber\\
=
&|\Omega|S(1+\log MS\delta^{-1})\sum^{N-1}_{k=0}(t_{k+1}-t_k)\max\{1,(T-t_k)^{-1}\}\frac{S(\frac{1}{\gamma}-\gamma)}{T-t_k+S\gamma}
\nonumber\\=
&|\Omega|S(1+\log MS\delta^{-1})S\left(\frac{1}{\gamma}-\gamma\right)\sum^{N-1}_{k=0}(t_{k+1}-t_k)\max\{1,(T-t_k)^{-1}\}\frac{1}{T-t_k+S\gamma}\nonumber\\
\lesssim&|\Omega|S^2(1+\log MS\delta^{-1})\left(\frac{1}{\gamma}-\gamma\right)\left(\log \frac{T+S\gamma}{1+S\gamma}+\frac{1}{S\gamma}\log\frac{\delta+S\gamma}{\delta(1+S\gamma)}\right),
\end{align}
where $(i)$ holds, since the symmetry yields the corresponding upper bound of $|\Delta_{t_k|y}|$ and this part of the score mismatch only occurs when $k\in\Omega$.

Hence, combining \eqref{eq:app:uni:bar} and \Cref{geral:unitau}, the proof is now complete.

\section{Proof of \texorpdfstring{\Cref{thm_klconstraint}}{Theorem V.3}}\label{app:thm_klconstraint}

To perform the convergence analysis, we first provide the following lemma to decompose the total error into the error of distribution of assignment $a$ and the error of the hard clamping reverse process.

\begin{lemma}\label{lemma_kldecomp}
   Let $v(a):=q_0(x_0^{\bar\Omega}=a|C),a\in A$ be the target assignment distribution and $w$ be any probability distribution on $A$. With the definition of $p^{\mathrm{RHC}}_{T-t}(x^\Omega|C)$ in \eqref{eq:def:rhcdistri}, we have
    \begin{align*}
        \mathrm{KL}(q_{\delta|C}^\Omega\|p_{T-\delta|C}^{\Omega,\mathrm{RHC}})\leq\mathrm{KL}(v\|w)+\mathbb{E}_{a\sim v}\left[\mathrm{KL}(q_{\delta|a}^\Omega\|p^{\Omega,\mathrm{HC}}_{T-\delta|a})\right],
    \end{align*}
    where $p_{T-\delta|C}^{\Omega,\mathrm{RHC}}:=p_{T-\delta}^{\mathrm{RHC}}(x^{\Omega}|C)$ refers to the distribution of RHC with set constraints and $p_{T-\delta|a}^{\Omega,\mathrm{HC}}:=p_{T-\delta}^{\mathrm{HC}}(x^{\Omega}|x_0^{\bar\Omega}=a)$ refers to the distribution that hard clamps the fixed observation to $a$.
\end{lemma}
\begin{proof}
    See \Cref{app:lemma_kldecomp}.
\end{proof}

\Cref{lemma_kldecomp} separates the total KL divergence under the set constraints setting into two parts. The second term is controlled by the fixed-observation convergence results provided in \Cref{sec:csedd}, while the first term $\mathrm{KL}(v\|w)$ quantifies the deviation between the alternative assignment distribution $w(a)$ and the target assignment distribution $v(a)$. Note that the mixture decomposition does not require an absorbing rate. 

To bound $\mathrm{KL}(v\|w)$, we start with $\mathrm{KL}(v\|w)+\mathrm{KL}(w\|v)$,
\begin{align}\label{eq:symkl}
    &\mathrm{KL}(v\|w)+\mathrm{KL}(w\|v)\nonumber\\
    =&\sum_av(a)\log\frac{v(a)}{w(a)}+\sum_aw(a)\log\frac{w(a)}{v(a)}\nonumber\\
    =&\sum_a(w(a)-v(a))\log\frac{w(a)}{v(a)}\nonumber\\
    =&\sum_a(w(a)-v(a))\log\frac{\prod_{j=1}^{|\bar\Omega|}\hat{s}_{T-t}(z_a^{(j)},z_a^{(j-1)})}{\prod_{j=1}^{|\bar\Omega|}s_{T-t}(z_a^{(j)},z_a^{(j-1)})}\nonumber\\
    &\qquad+\sum_a(w(a)-v(a))\log\frac{\sum_{b\in A}\prod_{j=1}^{|\bar\Omega|}s_{T-t}(z_b^{(j)},z_b^{(j-1)})}{\sum_{b\in A}\prod_{j=1}^{|\bar\Omega|}\hat{s}_{T-t}(z_b^{(j)},z_b^{(j-1)})}\nonumber\\
    \stackrel{(i)}{=}&\sum_a(w(a)-v(a))\log\frac{\prod_{j=1}^{|\bar\Omega|}\hat{s}_{T-t}(z_a^{(j)},z_a^{(j-1)})}{\prod_{j=1}^{|\bar\Omega|}s_{T-t}(z_a^{(j)},z_a^{(j-1)})}\nonumber\\
    =&\sum_a(w(a)-v(a))\sum_{j=1}^{|\bar\Omega|}\log\frac{\hat{s}_{T-t}(z_a^{(j)},z_a^{(j-1)})}{s_{T-t}(z_a^{(j)},z_a^{(j-1)})}\nonumber\\
    \leq&\sum_a(w(a)+v(a))\sum_{j=1}^{|\bar\Omega|}\left|\log\frac{\hat{s}_{T-t}(z_a^{(j)},z_a^{(j-1)})}{s_{T-t}(z_a^{(j)},z_a^{(j-1)})}\right|,
\end{align}
where $(i)$ holds since $\sum_a(w(a)-v(a))=0$.

When $x\geq1$, let $y=\log x\geq0$. By the Taylor expansion for $f(y)=e^y$, we have $e^y\geq1+y+\frac{1}{2}y^2$.  Then $x\geq 1+\log x+\frac{1}{2}(\log x)^2$, and thus $|\log x|\leq\sqrt{2(x-1-\log x)}$. 

When $0<x<1$, by the Taylor expansion for $f(x)=-\log x$, we have $-\log x\geq 1-x+\frac{1}{2}(x-1)^2$ and then $|\log x|=-\log x=x-1-\log x+(1-x)\leq x-1-\log x+\sqrt{2(x-1-\log x)}$.

Hence, $\forall x>0$, we have $|\log x|\leq \sqrt{2(x-1-\log x)}+(x-1-\log x)$, then 
\begin{align*}
    \left|\log \frac{\hat{s}_{T-t}(z_a^{(j)},z_a^{(j-1)})}{s_{T-t}(z_a^{(j)},z_a^{(j-1)})}\right|\leq &\sqrt{2\left(\frac{\hat{s}_{T-t}(z_a^{(j)},z_a^{(j-1)})}{s_{T-t}(z_a^{(j)},z_a^{(j-1)})}-1-\log \frac{\hat{s}_{T-t}(z_a^{(j)},z_a^{(j-1)})}{s_{T-t}(z_a^{(j)},z_a^{(j-1)})}\right)}\\&\qquad+\left(\frac{\hat{s}_{T-t}(z_a^{(j)},z_a^{(j-1)})}{s_{T-t}(z_a^{(j)},z_a^{(j-1)})}-1-\log \frac{\hat{s}_{T-t}(z_a^{(j)},z_a^{(j-1)})}{s_{T-t}(z_a^{(j)},z_a^{(j-1)})}\right).
\end{align*} 

Thus, we can obtain
\begin{align*}
    &\mathrm{KL}(v\|w)+\mathrm{KL}(w\|v)\\
    \stackrel{(ii)}{\leq}&\sum_a(w(a)+v(a))\sum_{j=1}^{|\bar\Omega|}\sqrt{2\left(\frac{\hat{s}_{T-t}(z_a^{(j)},z_a^{(j-1)})}{s_{T-t}(z_a^{(j)},z_a^{(j-1)})}-1-\log \frac{\hat{s}_{T-t}(z_a^{(j)},z_a^{(j-1)})}{s_{T-t}(z_a^{(j)},z_a^{(j-1)})}\right)}+2\varepsilon_{\mathrm{score}}^{\mathrm{path}}\\
    \stackrel{(iii)}{\leq}&\sqrt{\sum_a\sum_{j=1}^{|\bar\Omega|}(w(a)+v(a))}\\&\qquad\cdot\sqrt{\sum_a\sum_{j=1}^{|\bar\Omega|}2(w(a)+v(a))\left(\frac{\hat{s}_{T-t}(z_a^{(j)},z_a^{(j-1)})}{s_{T-t}(z_a^{(j)},z_a^{(j-1)})}-1-\log \frac{\hat{s}_{T-t}(z_a^{(j)},z_a^{(j-1)})}{s_{T-t}(z_a^{(j)},z_a^{(j-1)})}\right)}+2\varepsilon_{\mathrm{score}}^{\mathrm{path}}\\
    \leq&2\sqrt{2|\bar\Omega|\varepsilon_{\mathrm{score}}^{\mathrm{path}}}+2\varepsilon_{\mathrm{score}}^{\mathrm{path}},
\end{align*}
where $(ii)$ holds by \Cref{assp:rhc} and $(iii)$ holds by the Cauchy-Schwarz inequality.

Due to the non-negativity of KL divergence, we have
\begin{align*}
    \mathrm{KL}(v\|w)\leq&\mathrm{KL}(v\|w)+\mathrm{KL}(w\|v)\\
    \leq&2\sqrt{2|\bar\Omega|\varepsilon_{\mathrm{score}}^{\mathrm{path}}}+2\varepsilon_{\mathrm{score}}^{\mathrm{path}}\\
    \lesssim&\sqrt{|\bar\Omega|\varepsilon_{\mathrm{score}}^{\mathrm{path}}},
\end{align*}
where the last inequality holds since $0\leq\varepsilon_{\mathrm{score}}^{\mathrm{path}}\leq1$.

From Theorem \ref{lou:abtau} with absorbing rate and $\tau$-leaping, given one constraint value $a\in A$, we can have \begin{align*}
        \text{KL}(q_{\delta|a}^\Omega\|p^{\mathrm{\Omega,HC}}_{T-\delta|a})&\lesssim \mathcal{E}_{\mathrm{oracle}},
\end{align*}
where we take $\mathcal{E}_{\mathrm{oracle}}$ as the common bound for the fixed-observation sampling on $a\in A$.

Finally, combining all above together with \Cref{lemma_kldecomp}, we have 
\begin{align*}
        \text{KL}(q_{\delta|C}^\Omega\|p^{\mathrm{\Omega,RHC}}_{T-\delta|C})&\lesssim \sqrt{|\bar\Omega|\varepsilon_{\mathrm{score}}^{\mathrm{path}}}+
        \mathcal{E}_{\mathrm{oracle}}.
\end{align*}

\section{Difference between Hard Clamp Score and True Conditional Score}\label{example}

\begin{algorithm}[htbp]
\label{lousampler}
\caption{Conditional SEDD \cite{lou2023discrete}}
\SetKwInput{Input}{Input}
\SetKwInput{Output}{Output}

\Input{Initial sample $x_0 \sim p_0$, prompt coordinates $\bar\Omega$, observed value $y$ on $\bar\Omega$, time grid $0=t_0<t_1<...<t_N=T-\delta$, a learned unconditional score $\hat{s}_{T-t_n}$}
\For{$n=0$ \KwTo $N-1$}{
\For{$\hat{x}\neq x$ only at index $k$}{
    \eIf{$k\in\Omega$}{
            $\hat{Q}_{t_n|y}(x,\hat{x})=Q_{T-t_n}(\hat{x},x)\hat{s}_{T-t_n}\left([\hat{x}^\Omega,y],[x^\Omega,y]\right)$,

            where $[\cdot,y]$ refers to the whole sequence with $x^{\bar\Omega}=y$.
            
            }{
            $\hat{Q}_{t_n|y}(x,\hat{x})=0$
            }
}
Draw $x_{t_{n+1}}$ using $\hat{Q}_{t_n|y}$ (e.g., with $\tau$-leaping in \eqref{def:tau-leaping}).

Set $x_{t_{n+1}}^{\bar\Omega}=y$.
}
\end{algorithm}

The following binary example illustrates the difference between the hard clamp score and true conditional score. Consider a binary state space $[{0,1}]^2$ with $\Omega=\{1\}$ and $\bar\Omega=\{2\}$. Let the clean data distribution be
    \begin{align*}
        &q_0(x^1=0,x^2=0)=0.4,q_0(x^1=1,x^2=1)=0.4\\
        &q_0(x^1=0,x^2=1)=0.1,q_0(x^1=1,x^2=0)=0.1.
    \end{align*}
    Let the transition probabilities be $q_{t|0}(x_t^i=0|x_0^i=0)=\alpha_t$ and $q_{t|0}(x_t^i=0|x_0^i=1)=\beta_t$, $i\in\{1,2\}.$ Fix the condition $x_0^2=0$, then 
    \begin{align*}
        q_t(x_t^1=0|x_0^2=0)&=\sum_{x_0^1}q_t(x_t^1=0,x_0^1|x_0^2=0)\\
        &=\sum_{x_0^1} q_{t|0}(x_t^1=0|x_0^1, x_0^2=0)q_0^1(x_0^1|x_0^2=0)\\
        &=\sum_{x_0^1} q_{t|0}(x_t^1=0|x_0^1)q_0^1(x_0^1|x_0^2=0)\\
        &=0.8\alpha_t+0.2\beta_t.
    \end{align*}
Similarly,
\begin{align*}
        q_t(x_t^1=1|x_0^2=0)&=\sum_{x_0^1} q_{t|0}(x_t^1=1|x_0^1)q_0^1(x_0^1|x_0^2=0)\\
        &=1-0.8\alpha_t-0.2\beta_t.
    \end{align*}
    Now, consider conditioning at the intermediate time $t$,
    \begin{align*}
        &q_t(x^1=0|x^2=0)\\
        =&\frac{q_t(x=\{0,0\})}{q_t(x^2=0)}\\
        =&\frac{0.4q_{t|0}(x^1_t=0|x^1_0=0)q_{t|0}(x^2_t=0|x^2_0=0)+0.1q_{t|0}(x^1_t=0|x^1_0=0)q_{t|0}(x^2_t=0|x^2_0=1)}{\frac{1}{2}q_{t|0}(x^2_t=0|x^2_0=0)+\frac{1}{2}q_{t|0}(x^2_t=0|x^2_0=1)}\\
        &+\frac{0.1q_{t|0}(x^1_t=0|x^1_0=1)q_{t|0}(x^2_t=0|x^2_0=0)+0.4q_{t|0}(x^1_t=0|x^1_0=1)q_{t|0}(x^2_t=0|x^2_0=1)}{\frac{1}{2}q_{t|0}(x^2_t=0|x^2_0=0)+\frac{1}{2}q_{t|0}(x^2_t=0|x^2_0=1)}\\
        =&\frac{0.8(\alpha_t^2+\beta_t^2)+0.4\alpha_t\beta_t}{\alpha_t+\beta_t}.
    \end{align*}
Similarly, 
\begin{align*}
        &q_t(x^1=1|x^2=0)\\
        =&\frac{q_t(x=\{1,0\})}{q_t(x^2=0)}\\
        =&\frac{0.4q_{t|0}(x^1_t=1|x^1_0=0)q_{t|0}(x^2_t=0|x^2_0=0)+0.1q_{t|0}(x^1_t=1|x^1_0=0)q_{t|0}(x^2_t=0|x^2_0=1)}{\frac{1}{2}q_{t|0}(x^2_t=0|x^2_0=0)+\frac{1}{2}q_{t|0}(x^2_t=0|x^2_0=1)}\\
        &+\frac{0.1q_{t|0}(x^1_t=1|x^1_0=1)q_{t|0}(x^2_t=0|x^2_0=0)+0.4q_{t|0}(x^1_t=1|x^1_0=1)q_{t|0}(x^2_t=0|x^2_0=1)}{\frac{1}{2}q_{t|0}(x^2_t=0|x^2_0=0)+\frac{1}{2}q_{t|0}(x^2_t=0|x^2_0=1)}\\
        =&\frac{\alpha_t+\beta_t-0.8(\alpha_t^2+\beta_t^2)-0.4\alpha_t\beta_t}{\alpha_t+\beta_t}.
    \end{align*}
Choosing $x=\{0,0\}$ and $\hat{x}=\{1,0\}$, with uniform rate, we can obtain
\begin{align*}
    s_{t|y}(\hat{x},x)&=\frac{q_t(x_t^1=1|x_0^2=0)}{q_t(x_t^1=0|x_0^2=0)}=\frac{1-0.8\alpha_t-0.2\beta_t}{0.8\alpha_t+0.2\beta_t}=\frac{1-0.6e^{-t}}{1+0.6e^{-t}},\\
    s^{\text{hard clamp}}_{t|y}(\hat{x},x)&=\frac{q_t(x^1=1|x^2=0)}{q_t(x^1=0|x^2=0)}=\frac{\alpha_t+\beta_t-0.8(\alpha_t^2+\beta_t^2)-0.4\alpha_t\beta_t}{0.8(\alpha_t^2+\beta_t^2)+0.4\alpha_t\beta_t}=\frac{1-0.6e^{-2t}}{1+0.6e^{-2t}},
\end{align*}
since the transition probability $\alpha_t=\frac{1+e^{-t}}{2}$ and $\beta_t=\frac{1-e^{-t}}{2}$ for uniform rate.
Obviously, the hard clamp score does not coincide with the true conditional score.

\section{Detailed Settings of Experiments}\label{app:detexp}

In this section, we provide the detailed settings of the experiments in \Cref{main:exp}. 

\subsection{Conditional SEDD under Fixed Observations}\label{app:exp:consedd}
For synthetic data in \Cref{exp:seddnum}, we consider a finite discrete state space $[S]^d$ and construct a data distribution $q_0$ with $d=4$, $S=4$, where the alphabet is $\{0,1,2,3\}$. In detail, $q_0$ is constructed as a Gibbs distribution with token preferences sampled from the Gaussian distribution $\mathcal{N}(0,0.35^2)$ for each position. We also consider neighbor interactions, which include a $0.9$ reward for identical adjacent symbols and a $0.25$ reward for consecutive adjacent symbols. In absorbing rate simulations, we extend the state space with an additional MASK state and thus $S=5$.

We label the coordinates as $\{1,2,3,4\}$. The observed subset is $\bar\Omega=\{1,2\}$, and the unobserved subset is $\Omega=\{3,4\}$ for the fixed observation in \Cref{fig:absedd,fig:unisedd}. The observation of the first two coordinates is fixed to $y=[0,2]$. For the different fixed observations and sizes of free coordinates in \Cref{fig:differentunisedd}, the observations are fixed to $x_1=0$, $[x_1,x_2]=[0,2]$, and $[x_1,x_2,x_3]=[0,2,2]$. As a comparison, we also evaluate the observation $[x_1,x_2,x_3]=[0,2,1]$. Changing the observed vector from $[0,2,2]$ to $[0,2,1]$ results in $\gamma$ decreasing from approximately $0.02$ to $0.004$. 

The early stopping is not applied in all numerical experiments, and the time grid is set as $0=t_0<...<t_N=T$, where \[N\in\{16,32,64,128,256,512,1024,2048\},\quad T=5.\] We compare Conditional SEDD with the oracle conditional score baseline based on Monte-Carlo. At each time step and in each independent trial, we draw $10000$ samples and calculate the score ratios from the estimated probabilities.

\subsection{RHC under Set Constraints}\label{app:exp:rhc}
\paragraph{Synthetic Experiments}
For synthetic experiments in \Cref{exp:rhc}, we employ the same distribution, time grid, and sampling setting as those in \Cref{app:exp:consedd}. In \Cref{fig:rhcklresult}, we use $A_1=\{0,1\}$ and $A_2=\{2,3\}$ so that $A=A_1\times A_2$ and $C=\{x_1\in A_1,x_2\in A_2\}$. The assignment probability of each assignment is calculated by the unconditional scores based on \Cref{prop_distriequal}. 
\paragraph{Text Experiments}
For real-data experiments in \Cref{exp:rhc}, perplexity is selected as the major evaluation metric to measure the fluency and rationality of the generation results. A lower perplexity indicates that the generation results align more closely with the natural language, thereby achieving a better generation quality. We evaluate the perplexity of the full sequences using the pretrained GPT-2-large \cite{radford2019language}.

In detail, we use the pretrained discrete diffusion model SEDD-small in \cite{lou2023discrete} and tokenizer GPT-2 \cite{radford2019language}. The sampler used is Tweedie $\tau$-leaping with 1024 reverse sampling steps on a time grid from $t=1$ to $10^{-5}$. The sequence length is $d=50$, and the six manually constructed scenarios include transportation, study, game, cooking, project, and weather. Each scenario contains six constrained token positions. For each position, the sampler must select one token from the given feasible set. The feasible constraint sets are given in \Cref{app:realsetting}. Moreover, to ensure completeness for generated sentences, we fix the prefix to ``We'' for transportation, cooking, and weather; ``The team'' for game and project; and ``The student'' for study. The suffix is fixed to a period (``.'') in all scenarios. 

We also compare our RHC with a uniform-assignment baseline. For RHC, we use the pretrained unconditional scores to construct weights \[w(a)=\frac{\prod_{j=1}^{|\bar\Omega|}\hat{s}_t(z_a^{(j)},z_a^{(j-1)})}{\sum_{b\in A}\prod_{j=1}^{|\bar\Omega|}\hat{s}_t(z_b^{(j)},z_b^{(j-1)})}\] at time $t=1$. And the uniform-assignment baseline uses \[w_{\mathrm{uni}}(a)=\frac{1}{|A|}=\prod_{m\in\bar\Omega}\frac{1}{|A_m|},a\in A.\]For each scenario and method, we generate $1000$ sequences per trial and perform $5$ independent trials. All experiments are conducted on an NVIDIA A100 GPU.

\begin{table}[ht]
 \caption{Summary of scenarios.}
 \label{app:realsetting}
 \centering
 \begin{tabular}{lcc}
   \toprule
   Scenario & Position     & Candidate Tokens\\
   \midrule
   Transportation & 7 & plane, bus, taxi, bike, train, car \\
     & 15 & fast, quick, rapid, swift, speedy, brisk \\
     & 22 & anxious, nervous, worried, tense, uneasy, afraid \\
     & 27 & but, however, yet, while \\
     & 36 & safe, secure, protected, reliable, stable, harmless \\
     & 44 & calm, relaxed, peaceful, relieved, quiet, comfortable \\
     \midrule
   Game & 7 & game, match, race, tournament, contest, competition \\
     & 15 & tired, slow, weak, nervous, exhausted, weary \\
     & 22 & losing, trailing, struggling, failing, falling, slipping \\
     & 27 & but, however, yet, while \\
     & 36 & fought, scored, recovered, improved, rallied, persisted \\
     & 44 & won, succeeded, qualified, celebrated, advanced, prevailed \\
     \midrule
     Study & 7 & math, science, history, English, biology, physics \\
     & 15 & anxious, nervous, worried, tense, uneasy, afraid \\
     & 22 & hard, difficult, confusing, challenging, complex, tricky \\
     & 27 & but, however, yet, while \\
     & 36 & studied, practiced, reviewed, prepared, learned, revised \\
     & 44 & confident, calm, ready, relieved, focused, hopeful \\
     \midrule
     Cooking & 6 & soup, pasta, bread, rice, salad, pizza \\
     & 13 & boiled, baked, fried, grilled, roasted, smoked \\
     & 20 & fresh, warm, spicy, tasty, sweet, rich \\
     & 25 & but, however, yet, while \\
     & 33 & burned, cold, salty, bitter, raw, dry \\
     & 42 & fixed, served, shared, finished, saved, enjoyed \\
     \midrule
     Project & 6 & project, system, model, product, platform, application \\
     & 13 & difficult, complex, urgent, risky, challenging, demanding \\
     & 20 & broken, delayed, blocked, unstable, failing, incomplete \\
     & 25 & but, however, yet, while \\
     & 33 & tested, fixed, improved, rebuilt, updated, repaired \\
     & 42 & successful, stable, ready, reliable, complete, finished \\
     \midrule
     Weather & 6 & park, beach, forest, mountain, garden, river \\
     & 13 & sunny, warm, bright, clear, hot, pleasant \\
     & 20 & walking, hiking, running, camping, exploring, climbing \\
     & 25 & but, however, yet, while \\
     & 33 & rainy, cold, wet, cloudy, freezing, damp \\
     & 42 & returned, waited, rested, stopped, stayed, sheltered \\
   \bottomrule
 \end{tabular}
\end{table}

\begin{table}[ht]
 \caption{Examples of generated sentences. The bold text shows the selected constrained position.}
 \label{app:realresults1}
 \centering
 \begin{tabular}{lp{13cm}}
   \toprule
   Scenario & Generated Sentences\\
   \midrule
   Transportation & We started driving together in a single \textbf{car} and quickly became hooked on sharing this \textbf{fast} driving fantasy. We were both \textbf{nervous} at this point, \textbf{yet} aware it’s important to be \textbf{safe} in driver jobs. We got more \textbf{comfortable} working on this scenario. \\
   \cmidrule{2-2}
     & We’re looking for a \textbf{car}. We’ll make it \textbf{quick}. I’m not \textbf{worried} about driver hours, \textbf{but} confident we’ll be comfortable and \textbf{safe}. I’m \textbf{relieved} to make it in. \\
     \cmidrule{2-2}
     & We’re waiting for the \textbf{train}, and it’ll be \textbf{fast}. It’s a \textbf{tense}, terrifying experience, \textbf{but} almost everyone is enjoying it and everything is \textbf{safe}. Our focus right now is staying \textbf{calm} and enjoy the train.\\
     \midrule
   Game & The team then had a very difficult \textbf{match}. Our team manager made the fans \textbf{nervous}. The team reminded us of \textbf{losing} a great game, \textbf{but} this is it. The team rallied and \textbf{scored} our first goals. The team still \textbf{advanced} to the seventh round. \\
        \cmidrule{2-2}
        & The team took the field in the \textbf{match}. They played in a group with \textbf{weak} European play. They ended their \textbf{losing} streak in the match \textbf{but} they went through a difficult patch. They \textbf{fought} a very good campaign and the team \textbf{won} their four home matches.\\
        \cmidrule{2-2}
        & The team had competed in their early \textbf{race}. The team had suffered to a \textbf{slow} start, suffering the misfortune of \textbf{losing} in first race, \textbf{but} a few hours later, the team had \textbf{improved} to a higher level. They had \textbf{qualified} for their second race.\\
     \midrule
     Study & The student knew his teacher, his \textbf{math}, and his class. He was \textbf{nervous}, as he was taking a \textbf{hard} class with his teacher \textbf{but} was able to respond more quickly. He \textbf{prepared} for class. He knew he was \textbf{ready} to jump into class. \\
           \cmidrule{2-2}
           & The student intends to study somewhere between \textbf{English} and Pure English. But he is \textbf{afraid} that he will need to be \textbf{hard} and competitive. And \textbf{yet} it is important that he is good and \textbf{prepared}. It is important that he is \textbf{confident}, quick and decisive. \\
           \cmidrule{2-2}
           & The student really wanted to learn some \textbf{English}. She thought and felt a bit \textbf{nervous}. It might be even more \textbf{challenging} in the second exam \textbf{but} she was very confident. When the student \textbf{learned} some Danish and Norwegian, she was \textbf{ready} to access the textbooks. \\
   \bottomrule
 \end{tabular}
\end{table}

\begin{table}[ht]
 \caption{Examples of generated sentences. The bold text shows the selected constrained position.}
 \label{app:realresults2}
 \centering
 \begin{tabular}{lp{13cm}}
   \toprule
   Scenario & Generated Sentences\\
   \midrule
   Cooking & We had tossed green wasabi \textbf{salad} at sea. The combination of \textbf{roasted} pork and cilantro was very \textbf{spicy} and fragrant, \textbf{but} it varied a little bit in this \textbf{salty}-and-cheese dish. I \textbf{enjoyed} roasted pork and wasabi salad. \\
   \cmidrule{2-2}
     & We also share the grilled cheese \textbf{pizza}. This improved version of the \textbf{grilled} cheese pizza is fortified with the \textbf{sweet} and savory flavours \textbf{while} marinated with the plum and the \textbf{salty} onion, balanced by the juicy cheese and \textbf{served} with the tart and grape sticks. \\
     \cmidrule{2-2}
     & Weiner's sandwich is a \textbf{salad}. It tastes like a ' \textbf{smoked}' flavor of avocado, almost \textbf{spicy}, and sweet, \textbf{but} it does not have a ' spicy \textbf{salty}' flavor of onions. It is also \textbf{served} with hot and fried meat rolls.\\
     \midrule
   Project & The team first introduces a new \textbf{system}. The design is far more \textbf{complex}. The same system types are \textbf{broken} into primary domains, \textbf{while} differences in each case have to be \textbf{tested} separately. Once the system is \textbf{ready}, new features can be tested. \\
        \cmidrule{2-2}
        & The team managed to rebuild the \textbf{platform} and then solve a few more \textbf{complex} problems. The platform was very \textbf{unstable} from the beginning, \textbf{but} soon, its stability was increased and \textbf{improved}. The platform was once again in a \textbf{stable} condition and also with new challenges.\\
        \cmidrule{2-2}
        & The team working on the main \textbf{project} had sat for long periods with \textbf{difficult} problems. Work there had been \textbf{delayed} for several months, \textbf{but} after several months, the problems were \textbf{fixed}. The project was almost complete and was \textbf{ready} to be done before the holidays.\\
     \midrule
     Weather & We started off easy. The \textbf{forest} was green, and colors were \textbf{bright}. I knew where to go \textbf{hiking} so I could relax \textbf{while} enjoying the rest. Our road was \textbf{wet} and muddy which made me worry. I \textbf{stayed} with a friend and a backpack. \\
           \cmidrule{2-2}
           & We walk together at the dog \textbf{park}. It’s a \textbf{pleasant} experience you can hold onto while \textbf{walking} throughout the day, \textbf{but} I seemed to be getting cold and \textbf{wet} after a little bit of walking. I \textbf{stopped} at a light and started walking. \\
           \cmidrule{2-2}
           & We decided to camp at the \textbf{beach} and also try to be as \textbf{warm} as the beach. I was \textbf{camping} at my first beach \textbf{but} the weather was real cold and below \textbf{freezing} making it harder to camp here. I \textbf{stayed} and took a minute to rest. \\
   \bottomrule
 \end{tabular}
\end{table}

\section{Discussion of RHC with Uniform Rate}\label{app:uni:rhc}

In \Cref{sec:res:setcons}, we discuss the convergence analysis of RHC with absorbing rate. In this section, we introduce how to construct $v(a)$ with uniform rate. We demonstrate that the exact reconstruction of $v(a)$ still holds by the marginalization and forward transition kernel, but this construction introduces complicated computational and numerical challenges.

Recall that for uniform rate, the transition kernel is given by $q_{t|0}(x^i|x_0^i)=(1-e^{-t})/S$ when $x^i\neq x_0^i$; $(1+(S-1)e^{-t})/S$, otherwise. Hence, the forward transition kernel is invertible. 

Given a reference state $x^{\mathrm{ref}}\in [S]^d$, let $1\leq j\leq L_x$, then define a starting state $z_x^{(0)}=x^{\mathrm{ref}}$ and $z_x^{(j)}$ is obtained from $z_x^{(j-1)}$ by changing the $m_j$-th coordinate. Here, $L_x$ is the length of the chosen path from $z_x^0$ to $z_x^{L_x}$. For notational simplicity, we write $z_{(x^{\Omega},x^{\bar\Omega})}^{(j)}$ as $z_{x^{\bar\Omega}}^{(j)}$ and $L_{(x^{\Omega},x^{\bar\Omega})}$ as $L_{x^{\bar\Omega}}$. We can have the following proposition.
\begin{proposition}\label{app:prop:rhcuni}
    For each $a\in A$, with \textbf{uniform} rate, and the unconditional score $s_t(y,x)$, then, for any $t>0$, we can have
    \begin{align*}
        v(a)=q_0(x_0^{\bar\Omega}=a|C)=&\frac{\sum_{x^{\bar\Omega}}\left[\sum_{x^{\Omega}}\prod_{j=1}^{L_{x^{\bar\Omega}}}s_t(z_{x^{\bar\Omega}}^{(j)},z_{x^{\bar\Omega}}^{(j-1)})\right][(P_t^{-1})^{\otimes |\bar\Omega|}]_{x^{\bar\Omega},a}}{\sum_{b\in A}\sum_{x^{\bar\Omega}}\left[\sum_{x^{\Omega}}\prod_{j=1}^{L_{x^{\bar\Omega}}}s_t(z_{x^{\bar\Omega}}^{(j)},z_{x^{\bar\Omega}}^{(j-1)})\right][(P_t^{-1})^{\otimes |\bar\Omega|}]_{x^{\bar\Omega},b}}.
    \end{align*}
    Here, $P_t$ is the forward transition kernel for one coordinate.
\end{proposition}

Although \Cref{app:prop:rhcuni} provides the expression of the exact distribution $v(a)$, the computational cost of direct implementation may be significantly higher than that of absorbing rate. The proof is given in \Cref{app:prop_unirecon}. This reconstruction requires the summation over all states in $[S]^d$ and the matrix inversion uses the marginal over the entire constrained space $[S]^{\bar\Omega}$. Thus, it is impractical to enumerate these vectors. By comparison, although absorbing rate may still suffer from the high computational cost caused by the feasible set $A$, it avoids additional marginal and matrix inversion steps.

\section{Proof of \texorpdfstring{\Cref{prop_distriequal}}{Proposition V.1}}\label{app:prop_distriequal}

We begin with the transition probabilities with the absorbing rate. Recall that for one coordinate $k$, we have
\begin{align}
    q_{t|0}(x_t^k|y^k)&=\begin{cases}
        1-e^{-t} & \text{if}~x^k=\text{MASK}, y^k\neq \text{MASK}\\
        e^{-t} & \text{if}~x^k=y^k\neq \text{MASK}\\
        0 & \text{otherwise}
    \end{cases}.
\end{align}

Then, for a specific constraint value $a\in A$, we have
\begin{align*}
    q_t(x_t^{\Omega}=[\mathrm{MASK}]^{|\Omega|},x_t^{\bar\Omega}=a)=e^{-t|\bar\Omega|}(1-e^{-t})^{|\Omega|}\sum_{x_0^\Omega}q_0(x_0^{\Omega},x_0^{\bar\Omega}=a).
\end{align*}

Therefore,
\begin{align*}
    \frac{q_t(x_t^{\Omega}=[\mathrm{MASK}]^{|\Omega|},x_t^{\bar\Omega}=a)}{\sum_{b\in A}q_t(x_t^{\Omega}=[\mathrm{MASK}]^{|\Omega|},x_t^{\bar\Omega}=b)}
    =&\frac{e^{-t|\bar\Omega|}(1-e^{-t})^{|\Omega|}\sum_{x_0^\Omega}q_0(x_0^{\Omega},x_0^{\bar\Omega}=a)}{\sum_{b\in A}e^{-t|\bar\Omega|}(1-e^{-t})^{|\Omega|}\sum_{x_0^\Omega}q_0(x_0^{\Omega},x_0^{\bar\Omega}=b)}\\
    =&\frac{\sum_{x_0^\Omega}q_0(x_0^{\Omega},x_0^{\bar\Omega}=a)}{\sum_{b\in A}\sum_{x_0^\Omega}q_0(x_0^{\Omega},x_0^{\bar\Omega}=b)}\\
    =&q_0(x_0^{\bar\Omega}=a|C).
\end{align*}

Recall that the unconditional score is defined as $s_t(y,x):=\frac{q_t(y)}{q_t(x)}$, then
\begin{align*}
    \frac{q_t(x_t^{\Omega}=[\mathrm{MASK}]^{|\Omega|},x_t^{\bar\Omega}=a)}{\sum_{b\in A}q_t(x_t^{\Omega}=[\mathrm{MASK}]^{|\Omega|},x_t^{\bar\Omega}=b)}
    =&\frac{\frac{q_t(x_t^{\Omega}=[\mathrm{MASK}]^{|\Omega|},x_t^{\bar\Omega}=a)}{q_t(x_t=[\mathrm{MASK}]^d)}}{\frac{\sum_{b\in A}q_t(x_t^{\Omega}=[\mathrm{MASK}]^{|\Omega|},x_t^{\bar\Omega}=b)}{q_t(x_t=[\mathrm{MASK}]^d)}}\\
    =&\frac{\prod_{j=1}^{|\bar\Omega|}s_t(z_a^{(j)},z_a^{(j-1)})}{\sum_{b\in A}\prod_{j=1}^{|\bar\Omega|}s_t(z_b^{(j)},z_b^{(j-1)})}\\
    =&q_0(x_0^{\bar\Omega}=a|C)\\
    =&v(a),
\end{align*}
where for any $a\in A$, $z_a^{(0)}=[\mathrm{MASK}]^d$ and $z_a^{(j)}$ is obtained from $z_a^{(j-1)}$ by changing the $m_j$-th coordinate from $\mathrm{MASK}$ to $a^{m_j}$.
Now, we complete the proof.

\section{Proof of \texorpdfstring{\Cref{app:prop:rhcuni}}{Proposition XI.1}}\label{app:prop_unirecon}

We begin with the $q_t(x_t^{\bar\Omega}=b)$. For each $b\in [S]^{|\bar\Omega|}$, we have
\begin{align*}
    q_t(x_t^{\bar\Omega}=b)=&\sum_{x_t^{\Omega}}q_t(x_t^{\Omega},x_t^{\bar\Omega}=b)\\
    =&\sum_{x_t^{\Omega}}\sum_{x_0^{\Omega}}\sum_{x_0^{\bar\Omega}}q_{t|0}(x_t^{\Omega},x_t^{\bar\Omega}=b|x_0^{\Omega},x_0^{\bar\Omega})q_0(x_0^{\Omega},x_0^{\bar\Omega})\\
    =&\sum_{x_t^{\Omega}}\sum_{x_0^{\Omega}}\sum_{x_0^{\bar\Omega}}q_{t|0}(x_t^{\Omega}|x_0^{\Omega})q_{t|0}(x_t^{\bar\Omega}=b|x_0^{\bar\Omega})q_0(x_0^{\Omega},x_0^{\bar\Omega})\\
    =&\sum_{x_0^{\Omega}}\sum_{x_0^{\bar\Omega}}q_{t|0}(x_t^{\bar\Omega}=b|x_0^{\bar\Omega})q_0(x_0^{\Omega},x_0^{\bar\Omega})\\
    =&\sum_{x_0^{\bar\Omega}}q_{t|0}(x_t^{\bar\Omega}=b|x_0^{\bar\Omega})q_0(x_0^{\bar\Omega}).
\end{align*}

Rewriting the above equation in matrix form, we have $q_t^{\bar\Omega}=q_0^{\bar\Omega}P_t^{\otimes |\bar\Omega|}$ and thus $q_0^{\bar\Omega}=q_t^{\bar\Omega}(P_t^{-1})^{\otimes |\bar\Omega|}$, since $P_t$ is the forward transition kernel and invertible.

Now, given a reference state $x^{\mathrm{ref}}$, we denote $S_t(x):=\sum_{x^{\Omega}}\prod_{j=1}^{L_x}s_t(z_x^{(j)},z_x^{(j-1)})$, then
\begin{align*}
        S_t(b)=&\sum_{x^{\Omega}}\prod_{j=1}^{L_b}s_t(z_b^{(j)},z_b^{(j-1)})\\
        =&\sum_{x^{\Omega}}\frac{q_t(x^{\Omega},x_t^{\bar\Omega}=b)}{q_t(x^{\mathrm{ref}})}\\
        =&\frac{q_t(x_t^{\bar\Omega}=b)}{q_t(x^{\mathrm{ref}})}\\
        =&\frac{\sum_{x_0^{\bar\Omega}}q_{t|0}(x_t^{\bar\Omega}=b|x_0^{\bar\Omega})q_0(x_0^{\bar\Omega})}{q_t(x^{\mathrm{ref}})}.
\end{align*}

Again, in matrix form, we can have
$S_t=\frac{q_t^{\bar\Omega}}{q_t(x^{\mathrm{ref}})}=\frac{q_0^{\bar\Omega}P_t^{\otimes |\bar\Omega|}}{q_t(x^{\mathrm{ref}})}$, and thus $S_t(P_t^{-1})^{\otimes |\bar\Omega|}=\frac{q_0^{\bar\Omega}}{q_t(x^{\mathrm{ref}})}$. Now, for a given $a\in A$,
\begin{align*}
    \frac{q_0(x_0^{\bar\Omega}=a)}{q_t(x^{\mathrm{ref}})}=\sum_{x^{\bar\Omega}}S_t(x^{\bar\Omega})[(P_t^{-1})^{\otimes |\bar\Omega|}]_{x^{\bar\Omega},a}=\sum_{x^{\bar\Omega}}\left[\sum_{x^{\Omega}}\prod_{j=1}^{L_{x^{\bar\Omega}}}s_t(z_{x^{\bar\Omega}}^{(j)},z_{x^{\bar\Omega}}^{(j-1)})\right][(P_t^{-1})^{\otimes |\bar\Omega|}]_{x^{\bar\Omega},a}.
\end{align*}

Eventually, we can conclude that 
\begin{align*}
    v(a)=&q_0(x_0^{\bar\Omega}=a|C)\\
    =&\frac{q_0(x_0^{\bar\Omega}=a)}{\sum_{b\in A}q_0(x_0^{\bar\Omega}=b)}\\
    =&\frac{\sum_{x^{\bar\Omega}}\left[\sum_{x^{\Omega}}\prod_{j=1}^{L_{x^{\bar\Omega}}}s_t(z_{x^{\bar\Omega}}^{(j)},z_{x^{\bar\Omega}}^{(j-1)})\right][(P_t^{-1})^{\otimes |\bar\Omega|}]_{x^{\bar\Omega},a}}{\sum_{b\in A}\sum_{x^{\bar\Omega}}\left[\sum_{x^{\Omega}}\prod_{j=1}^{L_{x^{\bar\Omega}}}s_t(z_{x^{\bar\Omega}}^{(j)},z_{x^{\bar\Omega}}^{(j-1)})\right][(P_t^{-1})^{\otimes |\bar\Omega|}]_{x^{\bar\Omega},b}}.
\end{align*}

\section{Proof of Supporting Lemmas}\label{pro:lem}
\subsection{Proof of \texorpdfstring{\Cref{lem:pi}}{Lemma VII.1}}\label{app:proofpi}

We begin with $\pi(x_0^\Omega)=\prod_{i\in \Omega\setminus\{k\}}q_{T-t|0}(x^i|x_0^i)\sum_{x_0^{\bar\Omega}}\prod_{j\in \bar\Omega}q_{T-t|0}(y^j|x_0^j) q_0(x_0^\Omega,x_0^{\bar\Omega})$.
\begin{align}
    \pi(x_0^\Omega)=
    &\prod_{i\in \Omega\setminus\{k\}}q_{T-t|0}(x^i|x_0^i)\sum_{x_0^{\bar\Omega}}\prod_{j\in \bar\Omega}q_{T-t|0}(y^j|x_0^j) q_0(x_0^\Omega,x_0^{\bar\Omega})\nonumber\\\stackrel{(i)}{=}
    &\prod_{i\in \Omega\setminus\{k\}}q_{T-t|0}(x^i|x_0^i)\sum_{x_0^{\bar\Omega}}\left(\frac{1}{S}+\frac{S-1}{S}e^{-(T-t)}\right)^{|\bar{\Omega}|-U} \left(\frac{1}{S}-\frac{1}{S}e^{-(T-t)}\right)^{U}q_0(x_0^\Omega,x_0^{\bar\Omega})\nonumber\\=
    &\prod_{i\in \Omega\setminus\{k\}}q_{T-t|0}(x^i|x_0^i)\left(\frac{1}{S}+\frac{S-1}{S}e^{-(T-t)}\right)^{|\bar{\Omega}|} \nonumber\\&\qquad\qquad\qquad\qquad\cdot \sum_{x_0^{\bar\Omega}}\left(\frac{1-e^{-(T-t)} }{1+(S-1)e^{-(T-t)}}\right)^Uq_0(x_0^\Omega,x_0^{\bar\Omega})
    \nonumber\\\stackrel{(ii)}{=}
    &\prod_{i\in \Omega\setminus\{k\}}q_{T-t|0}(x^i|x_0^i)\left(\frac{1}{S}+\frac{S-1}{S}e^{-(T-t)}\right)^{|\bar{\Omega}|}\nonumber\\&\qquad\qquad\qquad\qquad\cdot\left[q_0(x_0^\Omega,y)+\sum_{x_0^{\bar\Omega}\neq y}\left(\frac{1-e^{-(T-t)} }{1+(S-1)e^{-(T-t)}}\right)^U q_0(x_0^\Omega,x_0^{\bar\Omega})\right](U\geq1)\label{pro:pilower}\\\leq
    &\prod_{i\in \Omega\setminus\{k\}}q_{T-t|0}(x^i|x_0^i)\left(\frac{1}{S}+\frac{S-1}{S}e^{-(T-t)}\right)^{|\bar{\Omega}|}\nonumber\\&\qquad\qquad\qquad\qquad\cdot\left[q_0(x_0^\Omega,y)+\sum_{x_0^{\bar\Omega}\neq y}\left(\frac{1-e^{-(T-t)} }{1+(S-1)e^{-(T-t)}}\right) q_0(x_0^\Omega,x_0^{\bar\Omega})\right]\nonumber\\=
    &\prod_{i\in \Omega\setminus\{k\}}q_{T-t|0}(x^i|x_0^i)\left(\frac{1}{S}+\frac{S-1}{S}e^{-(T-t)}\right)^{|\bar{\Omega}|}\nonumber\\&\qquad\qquad\qquad\qquad\cdot\left[q_0(x_0^\Omega,y)+\frac{1-e^{-(T-t)} }{1+(S-1)e^{-(T-t)}} (q_0(x_0^\Omega)-q_0(x_0^{\Omega},y))\right]\nonumber\\=
    &\left(\frac{1}{S}+\frac{S-1}{S}e^{-(T-t)}\right)^{|\bar{\Omega}|}\prod_{i\in \Omega\setminus\{k\}}q_{T-t|0}(x^i|x_0^i)q_0(x_0^\Omega,y)\nonumber\\&\qquad\qquad\qquad\qquad\cdot\left[1+\frac{1-e^{-(T-t)} }{1+(S-1)e^{-(T-t)}} \cdot\frac{q_0(x_0^\Omega)-q_0(x_0^{\Omega},y)}{q_0(x_0^\Omega,y)}\right]\nonumber\\\stackrel{(iii)}{=}
    &\left(\frac{1}{S}+\frac{S-1}{S}e^{-(T-t)}\right)^{|\bar{\Omega}|}\pi_{\text{true}}(x_0^{\Omega})\left[1+\frac{1-e^{-(T-t)} }{1+(S-1)e^{-(T-t)}} \cdot\frac{q_0(x_0^\Omega)-q_0(x_0^{\Omega},y)}{q_0(x_0^\Omega,y)}\right]\nonumber\\=
    &\left(\frac{1}{S}+\frac{S-1}{S}e^{-(T-t)}\right)^{|\bar{\Omega}|}\pi_{\text{true}}(x_0^{\Omega})\left[1+\frac{1-e^{-(T-t)} }{1+(S-1)e^{-(T-t)}} \left(\frac{1}{q_0(y|x_0^\Omega)}-1\right)\right].\nonumber
\end{align}
Here, for $(i)$, $|\bar\Omega|$ refers to the size of the observed set $\bar\Omega$ and $U$ refers to the number of unmatched tokens between $y$ and $x_0^{\bar\Omega}$. For $(ii)$, we extract the case $x_0^{\bar\Omega}=y$, and thus $U$ in the remaining terms is at least $1$. $(iii)$ holds since we have $\pi_{\text{true}}(x_0^\Omega):=\prod _{i\in {\Omega \backslash \{k\}}}q_{T-t|0}(x^i|x_0^i)q_0(x_0^{\Omega},x^{\bar\Omega}_0=y)$.

With \Cref{assp:gamma}, we can immediately have
 \begin{align*}
     \pi(x_0^{\Omega})&\leq \left(\frac{1}{S}+\frac{S-1}{S}e^{-(T-t)}\right)^{|\bar{\Omega}|}\left[1+\min\{1,T-t\} \left(\frac{1}{\gamma}-1\right)\right]\pi_{\text{true}}(x_0^{\Omega})\\&\leq \left(\frac{1}{S}+\frac{S-1}{S}e^{-(T-t)}\right)^{|\bar{\Omega}|}\frac{1}{\gamma}\pi_{\text{true}}(x_0^{\Omega}),
 \end{align*}
meanwhile, from \eqref{pro:pilower} with $\left(\frac{1-e^{-(T-t)} }{1+(S-1)e^{-(T-t)}}\right)^U q_0(x_0^\Omega,x_0^{\bar\Omega})$ nonnegative,
\begin{equation*}
    \pi(x_0^{\Omega})\geq\left(\frac{1}{S}+\frac{S-1}{S}e^{-(T-t)}\right)^{|\bar{\Omega}|}\pi_{\text{true}}(x_0^{\Omega}).
\end{equation*}
The proof is now complete.

\subsection{Proof of \texorpdfstring{\Cref{lemma_kldecomp}}{Lemma VIII.1}}\label{app:lemma_kldecomp}
First, we rewrite the distribution $q_t(x^\Omega|C)$ as follows:
\begin{align*}
    q_t(x^\Omega|C)=&\sum_{a\in A}q_t(x^\Omega,x_0^{\bar\Omega}=a|C)\\
    =&\sum_{a\in A}q_t(x^\Omega|x_0^{\bar\Omega}=a,C)q_0(x_0^{\bar\Omega}=a|C)\\
    \stackrel{(i)}{=}&\sum_{a\in A}q_t(x^\Omega|x_0^{\bar\Omega}=a)q_0(x_0^{\bar\Omega}=a|C)\\
    =&\sum_{a\in A}v(a)q_t(x^\Omega|x_0^{\bar\Omega}=a),
\end{align*}
where $(i)$ holds since $\{x_0:x_0^{\bar\Omega}=a\}\subseteq C$ for every $a\in A$.

Then recall that $p^{\mathrm{RHC}}_{T-\delta}(x^\Omega|C)=\sum_{a\in A}w(a)p^{\mathrm{HC}}_{T-\delta}(x^\Omega|x^{\bar\Omega}_0=a)$, and we have
\begin{align*}
    \mathrm{KL}(q^\Omega_{\delta|C}\|p_{T-\delta|C}^{\Omega,RHC})=&\sum_{x^\Omega\in[S]^{|\Omega|}}q_\delta(x^\Omega|C)\log\frac{q_\delta(x^\Omega|C)}{p^{\mathrm{RHC}}_{T-\delta}(x^\Omega|C)}\\
    =&\sum_{x^\Omega\in[S]^{|\Omega|}}\sum_{a\in A}v(a)q_\delta(x^\Omega|x_0^{\bar\Omega}=a) \log\frac{\sum_{b\in A}v(b)q_\delta(x^\Omega|x_0^{\bar\Omega}=b)}{\sum_{b\in A}w(b)p^{\mathrm{HC}}_{T-\delta}(x^\Omega|x^{\bar\Omega}_0=b)}\\
    \stackrel{(ii)}{\leq}&\sum_{a\in A}\sum_{x^\Omega\in[S]^{|\Omega|}}v(a)q_\delta(x^\Omega|x_0^{\bar\Omega}=a) \left(\log\frac{v(a)}{w(a)}+\log \frac{q_\delta(x^\Omega|x_0^{\bar\Omega}=a)}{p^{\mathrm{HC}}_{T-\delta}(x^\Omega|x^{\bar\Omega}_0=a)}\right)\\
     \stackrel{(iii)}{=}&\sum_{a\in A}v(a) \log\frac{v(a)}{w(a)}+\sum_{a\in A}v(a)\sum_{x^\Omega\in[S]^{|\Omega|}}q_\delta(x^\Omega|x_0^{\bar\Omega}=a)\log \frac{q_\delta(x^\Omega|x_0^{\bar\Omega}=a)}{p^{\mathrm{HC}}_{T-\delta}(x^\Omega|x^{\bar\Omega}_0=a)}\\
    =&\mathrm{KL}(v\|w)+\sum_{a\in A}v(a)\mathrm{KL}(q_{\delta|a}^\Omega\|p^{\mathrm{\Omega,HC}}_{T-\delta|a})\\
    =&\mathrm{KL}(v\|w)+\mathbb{E}_{a\sim v}\left[\mathrm{KL}(q_{\delta|a}^\Omega\|p^{\mathrm{\Omega,HC}}_{T-\delta|a})\right],
\end{align*}
where $(ii)$ is from the log-sum inequality over the assignments $a,b\in A$, and the following $(iii)$ holds since $\sum_{x^\Omega\in[S]^{|\Omega|}}q_\delta(x^\Omega|x_0^{\bar\Omega}=a)=1$. Now, we complete the proof.

\end{document}